\documentclass[sigconf]{acmart}
\acmConference{}{}{}
\renewcommand\footnotetextcopyrightpermission[1]{} 
\usepackage[utf8]{inputenc} 
\usepackage[T1]{fontenc}    
\usepackage{hyperref}       
\usepackage{url}            
\usepackage{booktabs}       
\usepackage{nicefrac}       
\usepackage{microtype}      
\usepackage{amsmath, amsfonts, amsthm}
\usepackage{xcolor}
\usepackage{tikz}
\usepackage{algorithm2e}
\usepackage{subcaption}
\usepackage{caption}
\usepackage{booktabs}
\usepackage{bm}
\usepackage{sidecap}
\usepackage{natbib}
\usepackage{pgf}
\usepackage{dblfloatfix}

\newtheorem{definition}{Definition}
\newtheorem{theorem}{Theorem}
\newtheorem{corollary}{Corollary}
\newtheorem{proposition}{Proposition}
\newtheorem{lemma}{Lemma}

\newtheorem{fact}{Fact}

\DeclareMathOperator*{\argmin}{arg\,min}

\SetKw{Break}{break}

\begin{document}

\title{Socially Fair \texorpdfstring{$k$}{k}-Means Clustering}

\author{Mehrdad Ghadiri}
\affiliation{%
  \institution{Georgia Tech}}
\email{ghadiri@gatech.edu}

\author{Samira Samadi}
\affiliation{%
  \institution{MPI for Intelligent Systems}
}
\email{ssamadi@tuebingen.mpg.de}

\author{Santosh Vempala}
\affiliation{%
\institution{Georgia Tech}}
\email{vempala@gatech.edu}

\allowdisplaybreaks

\begin{abstract}
We show that the popular $k$-means clustering algorithm (Lloyd's heuristic), used for a variety of scientific data, can result in outcomes that are unfavorable to subgroups of data (e.g., demographic groups). Such biased clusterings can have deleterious implications for human-centric applications such as resource allocation. We present a fair $k$-means objective and algorithm to choose cluster centers that provide equitable costs for different groups. The algorithm,  {\it Fair-Lloyd}, is a modification of Lloyd's heuristic for $k$-means, inheriting its simplicity, efficiency, and stability. In comparison with standard Lloyd's, we find that on benchmark datasets, Fair-Lloyd exhibits unbiased performance by ensuring that all groups have equal costs in the output $k$-clustering, while incurring a negligible increase in running time, thus making it a viable fair option wherever $k$-means is~currently~used.  
\end{abstract}

\maketitle
\thispagestyle{empty}

\section{Introduction}
\label{sec:intro}

Clustering, or partitioning data into dissimilar groups of similar items, is a core technique for data analysis. Perhaps the most widely used clustering algorithm is Lloyd's $k$-means heuristic \citep{steinhaus1956division,lloyd1982least,macqueen1967some}. 

Lloyd's algorithm starts with a random set of $k$ points (``centers'') and repeats the following two-step procedure:
(a) assign each data point to its nearest center; this partitions the data into $k$ disjoint groups (``clusters''); (b) for each cluster, set the new center to be the average of all its points. Due to its simplicity and generality, the $k$-means heuristic is widely used across the sciences, with applications spanning genetics \citep{krishna1999genetic}, image segmentation \citep{ray1999determination}, grouping search results and news aggregation \citep{sculley2010web},  crime-hot-spot detection \citep{grubesic2006application}, crime pattern analysis \citep{nath2006crime}, profiling road accident hot spots \citep{anderson2009kernel}, and market segmentation \citep{balakrishnan1996comparative}.

Lloyd's algorithm is a heuristic to minimize the $k$-means objective: choose $k$ centers such that the average squared distance of a point to its closest center is minimized. Note that, these $k$ centers automatically define a clustering of the data simply by assigning each point to its closest center. To better describe the $k$-means objective and the Lloyd's algorithm in the context of human-centric applications, let us consider two applications. 
In crime mapping and crime pattern analysis, law enforcement would run Lloyd's algorithm to partition areas of crime. This partitioning is then used as a guideline for allocating patrol services 
to each area (cluster). Such an assignment reduces the average 
response time of patrol units to crime incidents. A second application is market segmentation, where a pool of customers is partitioned using Lloyd's algorithm, and for each cluster, based on the customer profile of the center of that cluster, a certain set of services or advertisements is assigned to the customers in that cluster. 

In such human-centric applications, using the $k$-means algorithm in its original form, can
result in unfavorable and even harmful outcomes towards some demographic groups in the data. To illustrate bias, consider the Adult dataset from the UCI repository \citep{ucirepo}. This dataset consists of census information of individuals, including some sensitive attributes such as whether the individuals self identified as male or female. Lloyd's algorithm can be executed on this dataset to detect communities and eventually summarize communities with their centers.

Figure~\ref{fig:adultLloyd} shows the average $k$-means clustering cost for the Adult dataset \citep{ucirepo} for males vs females. The standard Lloyd's algorithm results in a clustering which incurs up to $15\%$ higher cost for females compared to males. Figure~\ref{fig:adultLloydRace} shows that this bias is even more noticeable among the five different racial groups in this dataset.
The average cost for an Asian-Pac-Islander individual is up to $4$ times worse than an average cost for a white individual. A similar bias can be observed in the Credit dataset \citep{yeh2009comparisons} between lower-educated and higher-educated individuals (Figure~\ref{fig:creditLloyd}).

\begin{figure*}[h!]
\centering
    \begin{subfigure}{.31\textwidth}
    \centering
  \includegraphics[width=\linewidth]{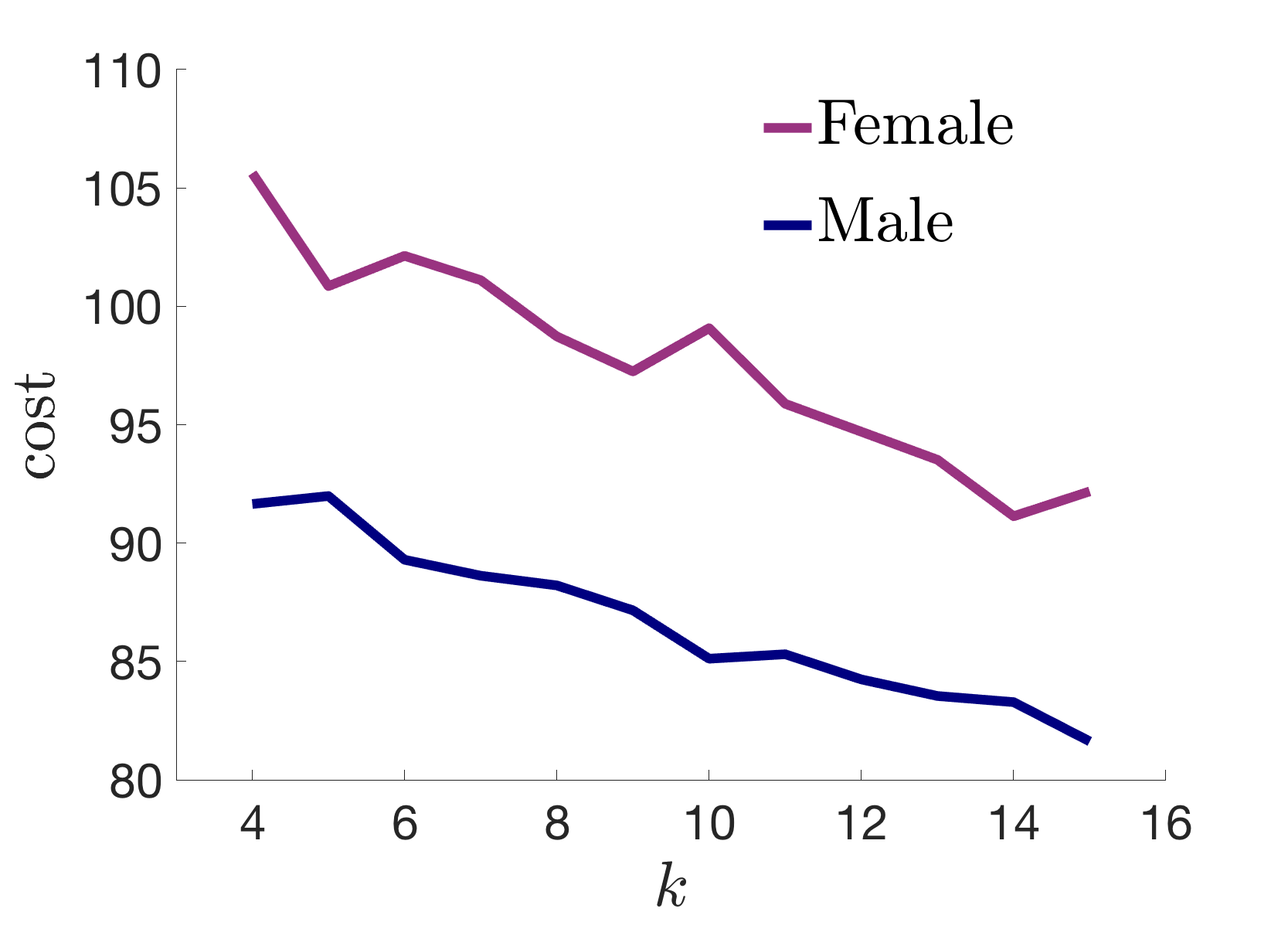}
  \subcaption{Adult census dataset}
  \label{fig:adultLloyd}
  \end{subfigure}
\begin{subfigure}{.31\textwidth}
    \centering
    \includegraphics[width=\linewidth]{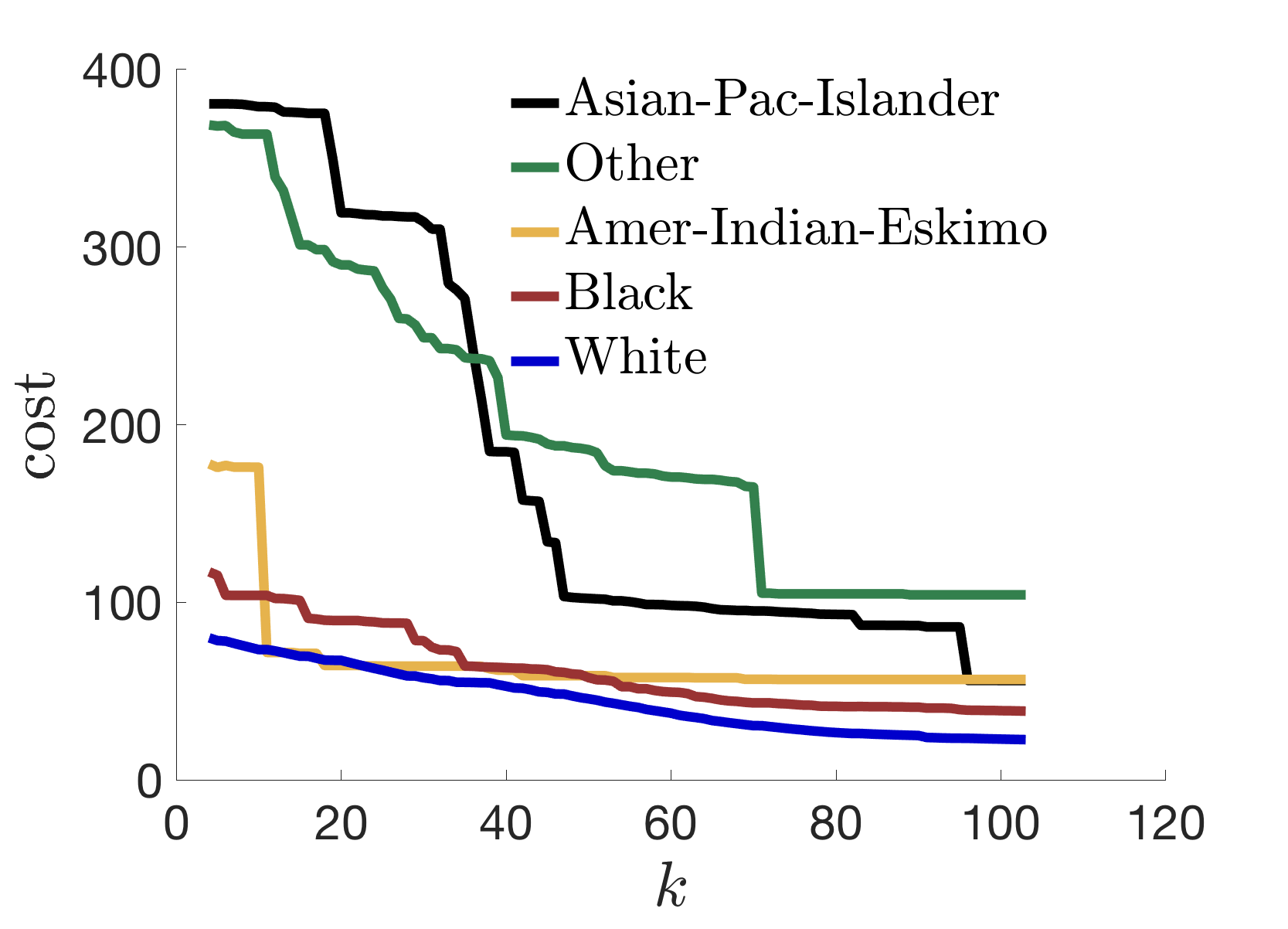}
    \subcaption{Adult census dataset}
    \label{fig:adultLloydRace}
\end{subfigure}
\begin{subfigure}{.31\textwidth}
    \centering
    \includegraphics[width=\linewidth]{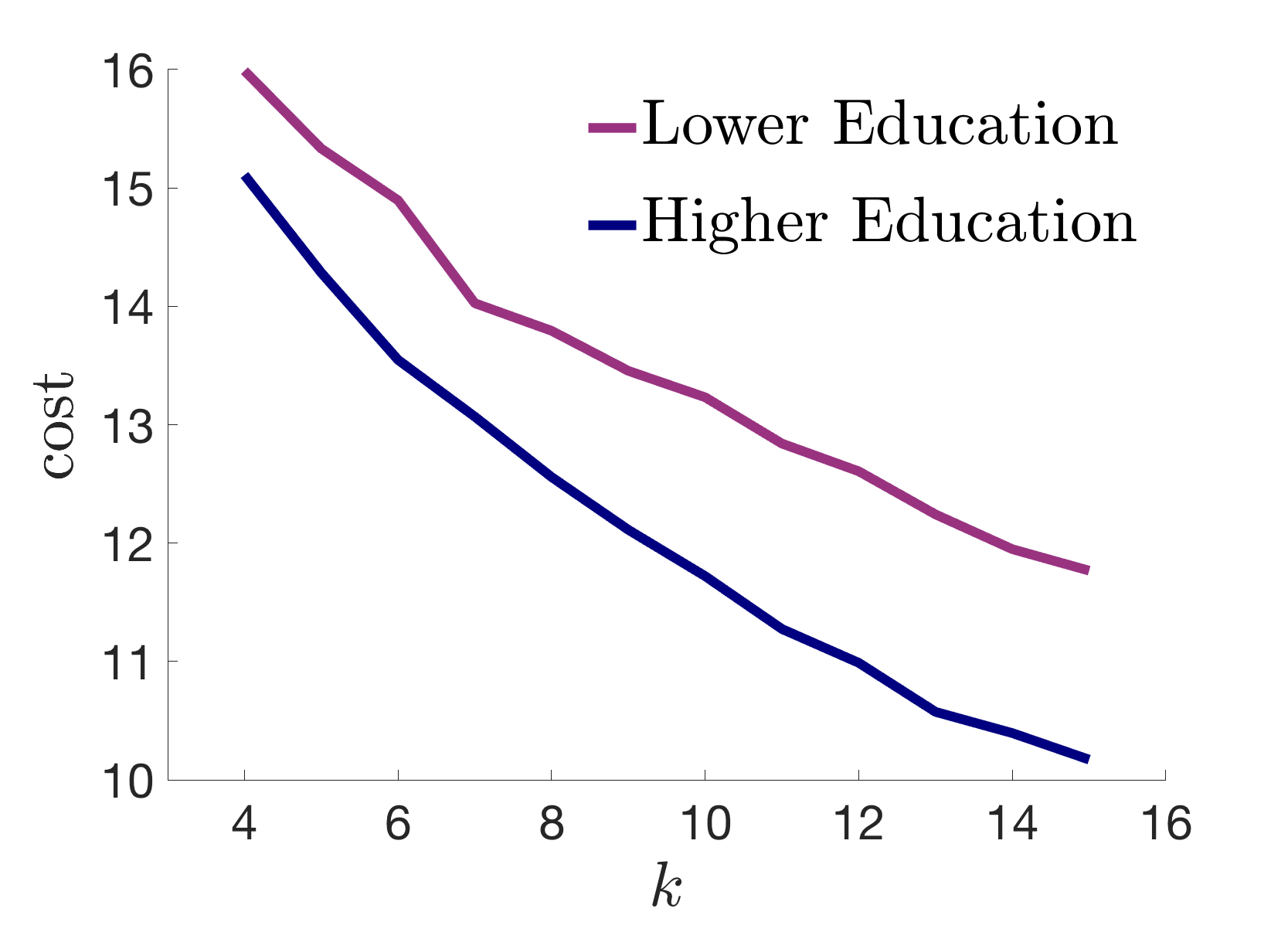}  
    \subcaption{Credit dataset}
    \label{fig:creditLloyd}
\end{subfigure}
    \caption{The standard Lloyd's algorithm results in a significant gap in the average clustering costs of different subgroups of the data.}
    \label{fig:motivation}
\end{figure*}

In this paper, we address the critical goal of {\em fair} clustering, i.e., a clustering whose cost is more equitable for different groups. This is, of course, an important and natural goal, and there has been substantial work on fair clustering, including for the $k$-means objective. Prior work has focused almost exclusively 
on {\em proportionality}, i.e., ensuring that sensitive attributes are distributed proportionally in each cluster~\cite{chierichetti2017fair,schmidt2019fair,huang2008labeled,backurs2019scalable,bera2019fair}.
In many application scenarios, including the ones illustrated above, one can view each setting of a sensitive attribute as defining a subgroup (e.g., gender or race), and the critical objective is the {\em cost} of the clustering for each subgroup: are one or more groups incurring a significantly higher average cost?  

In light of this consideration, we consider a different objective. 
Rather than minimizing the average clustering cost over the entire dataset, the objective of {\em socially fair} $k$-means is to find a $k$-clustering that minimizes the maximum of the average clustering cost across different (demographic) groups, i.e., minimizes the maximum of the average $k$-means objective applied to each group. \begin{center}
{\em    Can social fairness be achieved efficiently, while preserving the simplicity and generality of standard $k$-means algorithm?}
\end{center}
Applying existing algorithms for fair clustering with proportionality constraints leads to poor solutions for social fairness (see Figure ~\ref{fig:creditFairletComparison} for comparison on standard datasets), so we need a different solution. Our objective is similar to the recent line of work on minmax fairness through multi-criteria optimization \cite{samadi2018price,tantipongpipat2019multi,martinezminimax}.

\subsection{Our results}
We answer the above question affirmatively, with an algorithm we call {\em Fair-Lloyd}. Similar to Lloyd's algorithm, it is a two-step iteration with the only difference being how the centers are updated: (a) assign each data point to its nearest center to form clusters (b) choose $k$ new {\it fair} centers such that the maximum average clustering cost across different demographic groups is minimized. This step is particularly easy for $k$-means --- average the points in each cluster.
We prove that, the fair centers can also be computed efficiently: using
a simple one-dimensional line search when the data consists of two (demographic) groups, and using standard convex optimization algorithms when the data consists of more than two groups. Furthermore, when the data consists of two groups, the convergence of our algorithm is independent of the original dimension of the data and the number of clusters.

We prove convergence, stability and approximability guarantees and apply our method to multiple real-world clustering tasks. The results show clearly that Fair-Lloyd generates a clustering of the data with {\em equal} average clustering cost for individuals in different demographic groups. 
Moreover, its computational cost remains comparable to Lloyd's method. Each iteration, to find the next set of centers, is a convex optimization problem and can be implemented efficiently using Gradient Descent. For two groups, we give a line-search method which is significantly faster. This extends to a fast heuristic for $m>2$ groups whose distance to optimality can be tracked. This approach might be of independent interest as a very efficient heuristic for similar optimization problems. 

Due to the simplicity and efficiency of the Fair-Lloyd algorithm, we suggest it as an alternative to the standard Lloyd's algorithm in human-centric and other subgroup-sensitive applications where social fairness is a priority. 

\subsection{Fair k-means: Objective and Algorithm} 
To introduce the fair $k$-means objective, we define a more general notion: the $k$-means cost of a set of points $U$ with respect to a set of centers $C=\{c_1,\ldots,c_k\}$ and a partition $\mathcal{U}=\{U_1,\ldots,U_k\}$ of $U$ is
\vspace{-1mm}
\[
\Delta(C,\mathcal{U}):=\sum_{i=1}^k \sum_{p\in U_i} ||p-c_i||^2
\]
For a set of centers $C$, let $\mathcal{U}_C$ be a partition of $U$ such that if $p\in U_i$ then $\| p-c_i\|=\min_{1\leq j\leq k} \|p-c_j\|$. Then the standard $k$-means objective is 
\[
\min_{C=\{c_1,\ldots, c_k\}} \Delta(C,\mathcal{U}_C)
\]
\vspace{-1mm}
i.e., to find a set of $k$ centers $C=\{c_1,\ldots,c_k\}$ that minimizes $\Delta(C,\mathcal{U}_C)$.

For an illustrative example of the potential bias for different subgroups of data, see Figure~\ref{fig:motivation2} left. The two centers selected by minimizing the $k$-means objective are both close to one subgroup, and therefore the other subgroup has higher average cost. Note that the notion of fairness based on proportionality also prefers this clustering which impose a higher average cost on the purple subgroup. To introduce our fair $k$-means objective and algorithm, in this section we focus on the case of two (demographic) groups. In Section~\ref{sec:generalization}, we discuss how to generalize our framework to more than two groups.

\begin{figure}[h!]
\centering
    \begin{subfigure}{.49\columnwidth}
    \centering
    \caption*{$k$-means}
   \includegraphics[width=1.1\linewidth]{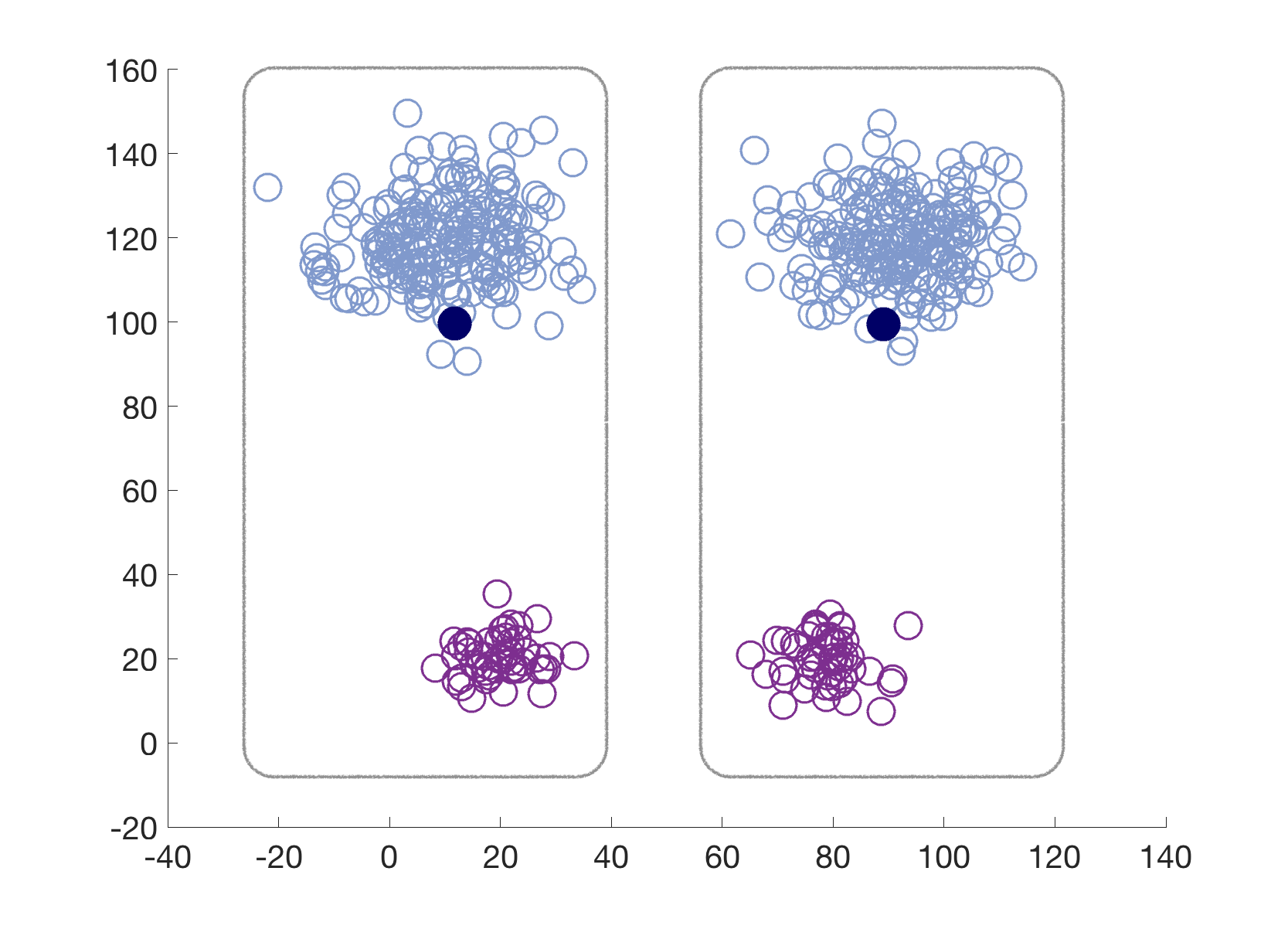}
   \label{subfig:kmeansexample}
   \end{subfigure}
    \begin{subfigure}{.49\columnwidth}
    \centering
    \caption*{Fair $k$-means}
    \includegraphics[width=1.1\linewidth]{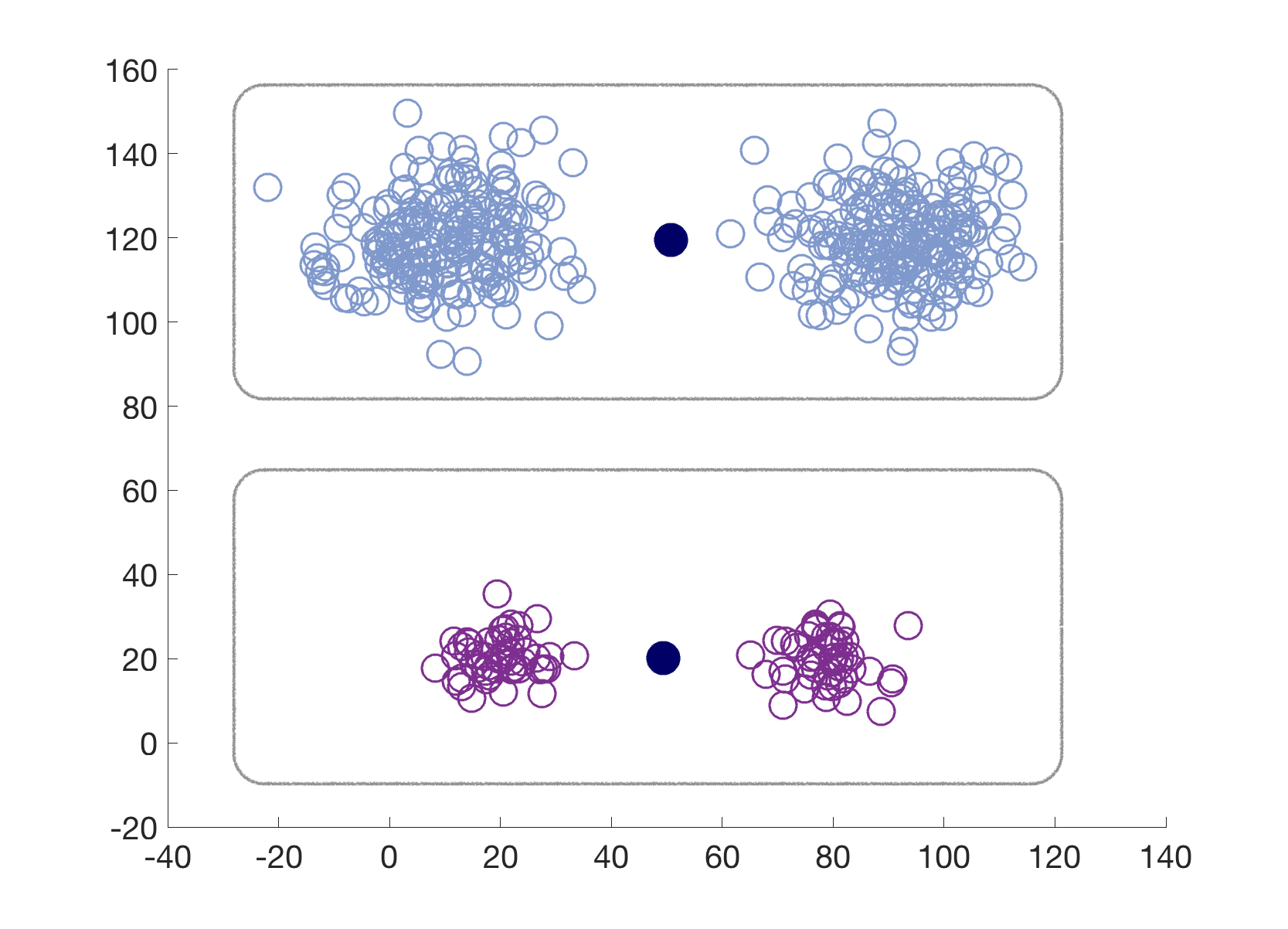} 
    \label{subfig:kmeansexamplefair}
\end{subfigure}
    
    \caption{
    Two demographic groups are shown with blue and purple. The $2$-means objective minimizing the average clustering cost
    prefers the clustering (and  centers) shown in the left figure. This clustering incurs a much higher average clustering cost for purple than for blue. The clustering in the right figure 
    has more equitable clustering cost for the two groups.}
    \label{fig:motivation2}
\end{figure}

The fair $k$-means objective for two groups $A,B$ such that $U = A\cup B$ is the larger average cost:
\[
\Phi(C, \mathcal{U}) := \max\{\frac{\Delta(C, \mathcal{U}\cap A)}{|A|}, \frac{\Delta(C,\mathcal{U}\cap B)}{|B|}\},
\]
where $\mathcal{U}\cap A=\{U_1\cap A,\ldots, U_k\cap A\}$.
The goal of fair $k$-means is to minimize $\Phi(C, \mathcal{U}_C)$, so as to minimize the \emph{higher} average cost. 
As illustrated in Figure~\ref{fig:motivation2} right, minimizing this objective results in a set of centers with equal average cost to individuals of different groups. In fact, as we will soon see, the solution to this problem equalizes the average cost of both 
groups
in most cases. Next we present the fair $k$-means algorithm (or Fair-Lloyd) in Algorithm~\ref{alg:fairlloyd}

\RestyleAlgo{algoruled}
\begin{algorithm}[ht]
\footnotesize
\textbf{Input:} A set of points $U=A\cup B$, and $k\in \mathbb{N}$
 	
 	Initialize the set of centers $C=\{c_1,\ldots,c_k\}$.
 	
 	\Repeat{convergence}{
 	1. Assign each point to its nearest center in $C$ to form a partition $\mathcal{U}=\{U_1,\ldots,U_k\}$ of $U$.
 	
 	2. Pick a set of centers $C$ that minimizes $\Phi(C,\mathcal{U})$. 
 	$$C \leftarrow \text{ Line Search}(U,\mathcal{U})$$
 	}
 	
 	\Return $C=\{c_1, \ldots, c_k\}$
\caption{Fair-Lloyd}
\label{alg:fairlloyd}
\end{algorithm}

The second step uses a minimization procedure to assign centers fairly to a given partition of the data. While this can be done via a gradient descent algorithm, 
we show in Section~\ref{sec:implementation} that it can be solved very efficiently using a simple line search procedure (see Algorithm~\ref{alg:binarysearch}) due to the structure of fair centers (Section~\ref{subsec:faircenters}). 
In Section~\ref{sec:miscellaneous}, we discuss some other properties of the fair $k$-means and Fair-Lloyd (Algorithm~\ref{alg:fairlloyd}). More specifically, we discuss the stability of the solution found by Fair-Lloyd, the convergence of Fair-Lloyd, and approximation algorithms that can be used for fair $k$-means
(e.g., to initialize the centers). In summary, our fair version of the $k$-means inherits its attractive properties while making the objective and outcome more equitable to subgroups of data.

\subsection{Related Work}
\label{sec:relatedWork}

\textbf{$k$-means objective and Lloyd's algorithm.} The $k$-means objective is NP-hard to optimize \citep{aloise2009np} and even NP-hard to approximate within a factor of $(1+\epsilon)$ \citep{awasthi2015hardness}. The best known approximation algorithm for the $k$-means problem finds a solution within a factor $\rho+\epsilon$ of optimal, where $\rho\approx 6.357$~\citep{AhmadianNSW17}. 
The running time of Lloyd's algorithm can  be exponential 
even on the plane \citep{vattani2011k}. 

As for the quality of the solution found, Lloyd's heuristic converges to a local optimum \citep{selim1984k}, with no worst-case guarantees possible \citep{kanungo2002local}. It has been shown that under certain assumptions on the existence of a sufficiently good clustering, this heuristic recovers a ground truth clustering and achieves a near-optimal solution to the $k$-means objective function~\citep{ostrovsky2013effectiveness,kumar2010clustering,awasthi2012improved}. For all the difficulties with the analysis, and although many other techniques has been proposed over the years, Lloyd's algorithm is still the most widely used clustering algorithm in practice \citep{jain2010data}.

\textbf{Fairness.} During the past years, machine learning has seen a huge body of work on fairness. Many formulations of fairness have been proposed for supervised learning and specifically for classification tasks \citep{dwork2012fairness,hardt2016equality,kleinberg2016inherent,zafar2015fairness}. The study of the implications of bias in unsupervised learning started more recently~\citep{chierichetti2017fair,celis2017ranking,celis2018fair,samadi2018price,schmidt2018fair,kleindessner19b}. We refer the reader to \cite{barocas-hardt-narayanan} for a summary of proposed definitions and algorithmic advances. 

Majority of the literature on fair clustering have focused on the proportionality/balance of the demographical representation inside the clusters \citep{chierichetti2017fair} --- a notion much in the nature of the widely known {\it disparate impact} doctrine. Proportionality of demographical representation has initially been studied for the $k$-center and $k$-median problems when the data comprises of two demographic groups \citep{chierichetti2017fair}, and later on for the $k$-means problem and for multiple demographic groups \cite{schmidt2019fair,huang2019coresets,bera2019fair,backurs2019scalable}. Among other notions of fairness in clustering, one could mention proportionality of demographical representation in the set of cluster centers \citep{kleindessner19b} or in large subsets of a cluster \citep{chen2019proportionally}. 

Our proposed notion of a fair clustering is different and comes from a broader viewpoint on fairness, aiming to enforce any objective-based optimization task to output a solution with {\it equitable objective value} for different demographic groups.
Such an objective-based fairness notion across subgroups could be defined subjectively e.g., by equalizing misclassification rate in classification tasks \citep{dieterich2016compas} or by minimizing the maximum error in dimensionality reduction or classification \citep{samadi2018price,tantipongpipat2019multi,martinezminimax}. We~define~a~{\it socially fair clustering} as the one that minimizes the maximum average clustering cost over different demographic groups. To the best of our knowledge, our work is the first to study fairness in clustering from this viewpoint.

\section{An Efficient Implementation of Fair \texorpdfstring{$k$}{k}-Means}
\label{sec:implementation}

The Fair-Lloyd algorithm (Algorithm~\ref{alg:fairlloyd}) is a two-step iteration, where the second step is to find a fair set of centers with respect to a partition. A set of centers $C^*$ is fair with respect to a partition  $\mathcal{U}$ if $C^*=\argmin_C \Phi(C,\mathcal{U})$.
In this section, we show that a simple line search algorithm can be used to find $C^*$ efficiently. 

\subsection{Structure of Fair Centers}
\label{subsec:faircenters}
 We start by illustrating some properties of fair centers. A partition of the data induces a partition of each of the two 
groups, and hence a set of means for each 
group
. Formally, for a set of points $U=A\cup B$ and a partition $\mathcal{U}=\{U_1,\ldots,U_k\}$ of $U$, let $\mu_i^A$ and $\mu_i^B$ be the mean of $A\cap U_i$ and $B\cap U_i$ respectively for $i\in [k]$. 
Our first observation is that the fair center of each cluster must be on the line segment between the means of the groups induced in the cluster.
\begin{lemma}
\label{lem:lineCenters}
Let $U=A\cup B$ and $ \mathcal{U} = (U_1, \ldots, U_k)$ be a partition of $U$.  
Let $C=(c_1, \ldots, c_k)$ be a fair set of centers with respect to $\mathcal{U}$. Then $c_i$ is on the line segment connecting $\mu_i^A$ and $\mu_i^B$. 
\end{lemma}
\vspace{-2mm}
\begin{proof}
For the sake of contradiction, assume that there exists an $i\in[k]$ such that $c_i$ is not on the line segment connecting $\mu_i^A$ and $\mu_i^B$. Note that (see \cite{kanungo2002local} for a proof of the following equation)

\begin{align*}
\sum_{p \in A\cap U_i} \|p-c_i\|^2 =  \sum_{p \in A\cap U_i} \|p-\mu_i^A\|^2 + |A\cap U_i| \|\mu_i^A-c_i\|^2 \\
\sum_{p \in B\cap U_i} \|p-c_i\|^2 =  \sum_{p \in B\cap U_i} \|p-\mu_i^B\|^2 + |B \cap U_i| \|\mu_i^B-c_i \|^2 
\end{align*}

Let $c_i'$ be the projection of $c_i$ to the line segment connecting $\mu_i^A$ and $\mu_i^B$. Then by Pythagorean theorem for convex sets, we have 
\begin{align*}
\|\mu_i^A - c_i\|^2 & \geq \|\mu_i^A - c_i'\|^2 + \|c_i'-c_i\|^2 \\
\|\mu_i^B - c_i\|^2 & \geq \|\mu_i^B - c'_i\|^2 + \|c_i'-c_i\|^2
\end{align*}

Therefore since $\| c_i'-c_i\|^2>0$, we have $\|\mu_i^A - c_i'\| < \|\mu_i^A - c_i\|$ and $\|\mu_i^B - c_i'\| < \|\mu_i^B - c_i\|$. Thus, replacing $c_i$ with $c'_i$ decreases the fair k-means objective. 
\end{proof}

The above lemma implies that, in order to find a fair set of centers, we only need to search the intervals $[\mu_i^A,\mu_i^B]$. 
Therefore we can find a fair set of centers by solving a convex program. The following definition will be convenient.

\begin{definition}
\label{def:notation}
Given $U=A\cup B$ and a partition $\mathcal{U}=\{U_1,\ldots,U_k\}$ of $U$, for $i=1,\ldots,k$, let 
\[
\alpha_i = \frac{|A\cap U_i|}{|A|}, \quad \beta_i = \frac{|B\cap U_i|}{|B|} \mbox{ and } 
l_i=\| \mu_i^A - \mu_i^B\|.
\]
Also let 
$
M^A=\{\mu_1^A,\ldots,\mu_k^A\} \mbox{ and } M^B=\{\mu_1^B,\ldots,\mu_k^B\}.
$
\end{definition}
We can now state the convex program.
\begin{corollary} 
Let $\mathcal{U} = \{U_1, \ldots, U_k\}$ be a partition of $U=A\cup B$. Then $C$ is a fair set of centers with respect to $\mathcal{U}$ if $c_i=\frac{(l_i-x^*_i)\mu_i^A+x^*_i\mu_i^B}{l_i}$, where $(x^*_1,\ldots,x^*_k,
\theta^*)$
is an optimal solution to the following convex program.
\begin{align}
\min ~ ~ & \label{eq:minCenterfirstline} \theta\\
\text{s.t. ~} & \frac{\Delta(M^A,\mathcal{U}\cap A)}{|A|} + 
\sum_{i \in [k]} \alpha_i {x_i}^2 \leq \theta \nonumber\\
& \frac{\Delta(M^B,\mathcal{U}\cap B)}{|B|} +
\sum_{i\in [k]} \beta_i  (l_i-x_i)^2 \leq \theta \nonumber \\
& 0 \leq x_i \leq l_i \hspace{4mm},  i \in [k] \label{eq:minCenterlastline} \nonumber
\end{align}
\end{corollary}

We can solve this convex program with standard convex optimization methods such as gradient descent. However, as we show in the next section, we can solve it with a much faster algorithm.

\subsection{Computing Fair Centers via Line Search}
\label{subsec:linesearch}
We first need to review a couple of facts about subgradients. For a convex continuous function $f$, we say that a vector $u$ is a subgradient of $f$ at point $x$ if $f(y)\geq f(x) + u^T(y-x)$ for any $y$. We denote the set of subgradients of $f$ at $x$ by $\partial f(x)$.

\begin{fact}
\label{fact:0subgrad}
Let $f$ be a convex function. Then point $x^*$ is a minimum for $f$ if and only if $\vec{0}\in\partial f(x^*)$.
\end{fact}

\begin{fact}
\label{fact:maxsubgrad}
Let $f_1,\ldots,f_m$ be smooth functions and \[F(x)=\max_{j\in [m]} f_j(x).\] Let $S_x=\{j\in[m]: f_j(x)=F(x)\}$. Then the set of subgradients of $F$ at $x$ is the convex hull of union of the subgradients of $f_j$'s at $x$ for $j\in S_x$.
\end{fact}
Let 
\begin{align*}
f_A(x) & :=\frac{\Delta(M^A,\mathcal{U}\cap A)}{|A|} + 
\sum_{i \in [k]} \alpha_i {x_i}^2, \\ f_B(x) & :=\frac{\Delta(M^B,\mathcal{U}\cap B)}{|B|} +
\sum_{i\in [k]} \beta_i  (l_i-x_i)^2.
\end{align*}

Then we can view the convex program (\ref{eq:minCenterfirstline}) as minimizing 
\begin{equation}
\label{eq:phi}
f(x):=\max\{f_A(x), f_B(x)\} \mbox{ s.t. } 0 \leq x_i \leq l_i, \forall i\in[k]. 
\end{equation}
Note that $f(x)$ is convex since the maximum of two convex functions is convex. Therefore by Fact~\ref{fact:0subgrad}, our goal is to find a point $x^*$ such that $\vec{0}\in\partial f(x^*)$. Note that $f_A$ and $f_B$ are differentiable. Hence by Fact~\ref{fact:maxsubgrad}, we only need to look at points $x$ for which there exists a convex combination of $\nabla f_A(x)$ and $\nabla f_B(x)$ that is equal to $\vec{0}$. As we will see, this set of points is only a one-dimensional curve in $[0,l_1]\times\cdots\times[0,l_k]$. When $f_A(x)>f_B(x)$, $f(x)$ has a unique gradient and it is equal to $\nabla f_A(x)$. Similarly, when $f_A(x)<f_B(x)$, we have $\nabla f(x)=\nabla 
f_B(x)$.
In the case that $f_A(x)=f_B(x)$, for any $\gamma\in[0,1]$, 
\[
u(\gamma,x):=\gamma \nabla f_A(x)+(1-\gamma)\nabla f_B(x)
\]
is a subgradient of $f(x)$ --- and these are the only subgradient of $f$ at $x$. Now consider the set 
\[
Z=\{x\in [0,l_1]\times\cdots\times[0,l_k]: \exists \gamma\in [0,1], u(\gamma,x)=\vec{0}\}.
\]
If we find $x^*\in Z$ such that $ f_A(x^*)= f_B(x^*)$, then $\vec{0}\in\partial f(x^*)$ and therefore, $x^*$ is an optimal solution. We first describe $Z$ and show that there exists an optimal solution in $Z$. 

\begin{lemma}
\label{lem:structure}
Let $\gamma\in[0,1]$ and $u(\gamma,x)=\vec{0}$. Then $x_i=\frac{(1-\gamma)\beta_i l_i}{\gamma \alpha_i + (1-\gamma) \beta_i}$.
\end{lemma}

\begin{proof}
We have $\frac{\partial}{\partial x_i} f_A(x)=2\alpha_i x_i$ and $\frac{\partial}{\partial x_i}f_B(x)=2\beta_i (x_i-l_i)$. Using the fact that $u(\gamma,x)=\vec{0}$, we have 
\[
\gamma (2\alpha_i x_i) + (1-\gamma)(2\beta_i (x_i-l_i))=0.
\]
Hence, 
\[
x_i = \frac{(1-\gamma)\beta_i l_i}{\gamma \alpha_i + (1-\gamma) \beta_i}.
\]
\end{proof}

The previous lemma gives a complete description of set $Z$. One example of set $Z$ is shown in Figure~\ref{fig:curveGamma} left for the case of $k=2$. The following is an immediate result of Lemma~\ref{lem:structure}.
\begin{lemma}
$Z=\{x: x_i=\frac{(1-\gamma)\beta_i l_i}{\gamma \alpha_i + (1-\gamma) \beta_i}, \gamma\in[0,1]\}$.
\end{lemma}

\begin{figure*}[ht!]
\begin{subfigure}{.4\textwidth}
  \centering
  \includegraphics[width=\linewidth]{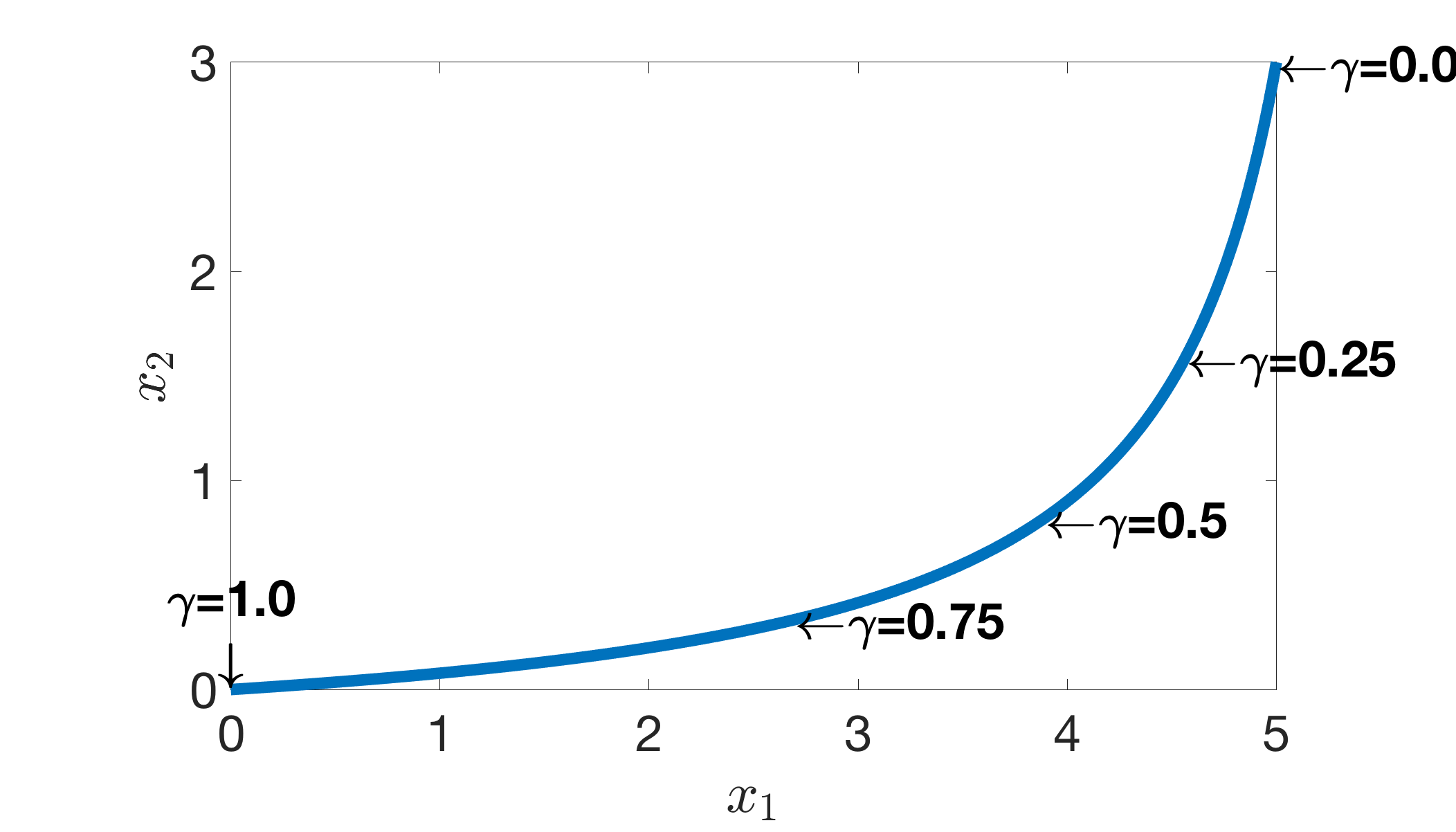}  
  \label{fig:sub-curve}
\end{subfigure}
\begin{subfigure}{.4\textwidth}
  \centering
  \includegraphics[width=\linewidth]{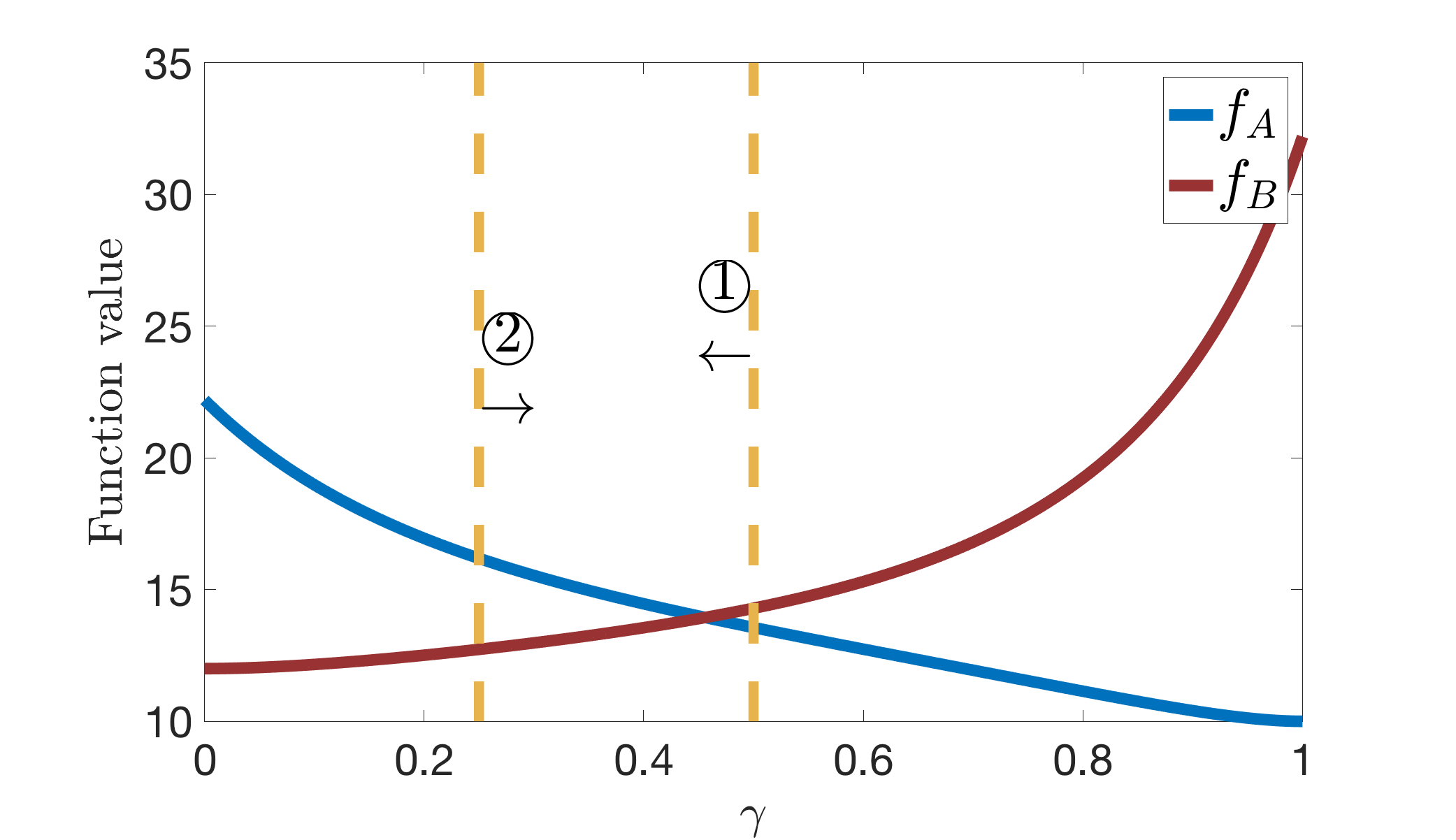}  
  \label{fig:sub-fggamma}
\end{subfigure}
\caption{Left: an example of the one-dimensional curve for $k=2$. Right: the functions $f_A$ and $f_B$ with respect to $\gamma$, and two steps of the line search algorithm. We can use a line search to find the optimal value of $\gamma$ and an optimal solution to (\ref{eq:phi}).}
\label{fig:curveGamma}
\end{figure*}

Note that when $\gamma=1$, this recovers the all-zero vector and when $\gamma=0$, $x=(l_1,\ldots,l_k)$. Therefore these extreme points are also in $Z$. As we mentioned before if there exists $x^*\in Z$ such that $f_A(x^*)=f_B(x^*)$, then $x^*$ is an optimal solution. Therefore suppose such an $x^*$ does not exist. One can see that 
\[
\frac{d}{d\gamma}(\frac{(1-\gamma)\beta_i l_i}{\gamma \alpha_i + (1-\gamma) \beta_i}) = \frac{-\alpha_i \beta_i l_i}{(\gamma \alpha_i + (1-\gamma) \beta_i)^2}.
\]
Therefore, for $1\leq i\leq k$, $x_i$ is decreasing in $\gamma$. Also one can see that $f_A$ is increasing in $x_i$ and $f_B$ is decreasing in $x_i$. Therefore $f_A$ is decreasing in $\gamma$ and $f_B$ is increasing in $\gamma$. Figure~\ref{fig:curveGamma} right shows an example that illustrates the change of $f_A$ and $f_B$ with respect to $\gamma$. This implies that if there does not exist any $x^*\in Z$ such that $f_A(x^*)=f_B(x^*)$, then either $f_A(\vec{0})>f_B(\vec{0})$ or $f_B(\ell)>f_A(\ell)$, where $\ell=(l_1,\ldots,l_k)$. In the former case, the optimal solution to (\ref{eq:minCenterfirstline}) is $\vec{0}$ which means that the fair centers are located on the means of points for 
group
$A$. In the latter case the optimal solution is $\ell$ which means that the fair centers are located on the means of points for 
group
$B$.

The above argument asserts that 
we only need to search the set $Z$ to find an optimal solution. Each element of $Z$ is uniquely determined by the corresponding $\gamma\in[0,1]$. Our goal is to find an element $x^*\in Z$ such that $f_A(x^*)=f_B(x^*)$. Since $f_A$ is decreasing in $\gamma$ and $f_B$ is increasing in $\gamma$, we can use line search to find such a point in $Z$. If such a point does not exist in $Z$, then the line search converges to $\gamma=0$ or $\gamma=1$. Two steps of such a line search are shown in Figure~\ref{fig:curveGamma}. See Algorithm~\ref{alg:binarysearch} for a precise description. Using this line search algorithm, we can solve the convex program described in (\ref{eq:phi}) in $O(k\log (\max_i l_i))$ time. 
\vspace{-1.5mm}

\RestyleAlgo{algoruled}
\begin{algorithm}[th]
\footnotesize
\textbf{Input:} A set of points $U=A\cup B$ and a partition $\mathcal{U}=\{U_1,\ldots,U_k\}$ of $U$.\\[0.6ex]

Compute $\alpha_i, \beta_i, \mu_i^A, \mu_i^B, l_i, M^A, M^B$
\DontPrintSemicolon\tcp*{See Definition~\ref{def:notation}.}

$\gamma\leftarrow 0.5$

\For{$t = 1,\ldots, T$}{
	$x_i \leftarrow \frac{(1-\gamma)\beta_i l_i}{\gamma \alpha_i + (1-\gamma) \beta_i}$, for $i=1,\ldots k$
	
	Compute $f_A(x)$ and $f_B(x)$
	

    \uIf{$f_A(x)>f_B(x)$}{
        $\gamma\leftarrow \gamma + (1/2)^{-(t+1)}$
    }\uElseIf{$f_A(x)<f_B(x)$}{
        $\gamma\leftarrow \gamma - (1/2)^{-(t+1)}$
    }\Else{
        \Break
    }
    
}

$c_i\leftarrow\frac{(l_i-x_i)\mu_i^A+x_i\mu_i^B}{l_i}$, for all $i=1,\ldots,k$

\Return $C=(c_1,\ldots,c_k)$
\caption{Line Search$(U,\mathcal{U})$}
\label{alg:binarysearch}
\end{algorithm}
\vspace{-2mm}
\subsection{Fair \texorpdfstring{$k$}{k}-means is well-behaved}
\label{sec:miscellaneous}
In this section, we discuss the stability, convergence, and approximability of Fair-Lloyd for 2 groups. As we will show in Section~\ref{sec:generalization}, these results can be extended to $m>2$ groups. 

\paragraph{Stability.} The line search algorithm finds the optimal solution to (\ref{eq:minCenterfirstline}). This means that for a fixed partition of the points (e.g., the last clustering that the algorithm outputs), the returned centers are optimal in terms of the maximum average cost of the
groups. However, one important question is whether we can improve the cost for the 
group
with the smaller average cost.
The following proposition shows that this is not possible; assuring that the solution is pareto optimal. 

\begin{proposition}
\label{prop:stability}
Let $x^*$ be the optimal solution for a fixed partition $\mathcal{U}=\{U_1,\ldots,U_k\}$ of $U$. Then there does not exist any other optimal solution with an average cost better than $f_A(x^*)$ or $f_B(x^*)$ for 
groups
$A$ and $B$, respectively.
\end{proposition}

\begin{proof}
Let $y$ be another optimal solution. Without loss of generality, suppose $f_A(x^*)\geq f_B(x^*)$. If $f_A(x^*)> f_B(x^*)$, then by our discussion on the line search algorithm, $x^*=\vec{0}$ and it is the only optimal solution. Therefore $y=x^*$. Now suppose $f_A(x^*)= f_B(x^*)$. For the sake of contradiction and without loss of generality, assume $f_A(x^*)=f_A(y)$, but $f_B(x^*)>f_B(y)$. Therefore $f_A(y)>f_B(y)$. First note that $y\neq \vec{0}$ because if $f_A(\vec{0})>f_B(\vec{0})$, then for any other $x$ in the feasible region, $f_A(x)>f_B(x)$ which is a contradiction because $f_A(x^*)= f_B(x^*)$. Hence we can decrease one of the coordinates of $y$ by a small amount to get a point $y'$ in the feasible region. If the change is small enough, we have $f_A(y)>f_A(y')>f_B(y')>f_B(y)$ but this is a contradiction because it implies $f(x^*)=f(y)>f(y')$ which means $x^*$ was not an optimal solution.
\end{proof}

\paragraph{Convergence.}
Lloyd's algorithm for the standard $k$-means problem converges to a solution in finite time, essentially because the number of possible partitions is finite~\citep{lloyd1982least}. This also holds for the Fair-Lloyd algorithm for the fair $k$-means problem. Note that for any fixed partition of the points, our algorithm finds the optimal fair centers. Also, note that there are only a finite number of partitions of the points. Therefore, if our algorithm continues until a step where the clustering does not change afterward, then we say that the algorithm has converged and indeed the solution is a local optimum. However, note that, in the case where there is more than one way to assign points to the centers (i.e., there exists a point that have more than one closest center), then we should exhaust all the cases, otherwise the output is not necessarily a local optimal. This is not a surprise because the same condition also holds for the Lloyd's algorithm for the $k$-means problem. For example, see Figure~\ref{fig:kmeansCE}. Adjacent points have unit distance from each other. The centers are optimum for the illustrated clustering. However, they do not form a local optimum because moving $c_2$ and $c_3$ to the left by a small amount $\epsilon$ decreases the $k$-means objective from $2$ to $2(1-\epsilon)^2+\epsilon^2=2-4\epsilon+3\epsilon^2 < 2$.

\begin{figure}[h]
    \centering
    \begin{tikzpicture}
\draw[blue!50, very thick, dashed] (0,0) ellipse (1.25 and 1);
\draw[blue!50, very thick, dashed] (-2,0) ellipse (0.5 and 1);
\draw[blue!50, very thick, dashed] (2,0) ellipse (0.5 and 1);
\filldraw[black] (0,0) circle (2pt) node[anchor=north] {$c_2$};
\filldraw[black] (-1,0) circle (2pt) node[anchor=south] {$A$};
\filldraw[black] (1,0) circle (2pt) node[anchor=south] {$A$};
\filldraw[black] (0,0) circle (2pt) node[anchor=north] {$c_2$};
\filldraw[black] (-2,0) circle (2pt) node[anchor=north] {$c_1$};
\filldraw[black] (2,0) circle (2pt) node[anchor=north] {$c_3$};
\filldraw[black] (-2,0) circle (2pt) node[anchor=south] {$A$};
\filldraw[black] (2,0) circle (2pt) node[anchor=south] {$A$};
\end{tikzpicture}
    \caption{An example of $k$-means problem where the current clustering is not a local optimal and we need to check all the possible partitions with the current centers. $c_1,c_2,c_3$ are the centers and the points are marked with the letter A on top of them.}
    \label{fig:kmeansCE}
\end{figure}
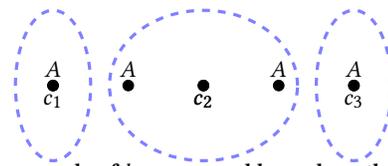
\vspace{-2mm}
\paragraph{Initialization.} An important consideration is how to initialize the centers. While a random choice is often used in practice for the $k$-means algorithm, another choice that has better provable guarantees~\citep{kumar2010clustering} is to use a set of centers with objective value that is within a constant factor of the minimum. We will show that a $c$-approximation for the $k$-means problem implies a $2c$-approximation for the fair $k$-means problem, and so this method could be used to initialize centers for Fair-Lloyd as well.
The best known approximation algorithm for the $k$-means problem finds a solution within a factor $\rho+\epsilon$ of optimal, where $\rho\approx 6.357$~\citep{AhmadianNSW17}.

\begin{theorem}
\label{thm:approx}
If the $k$-means problem admits a $c$-approximation in polynomial time then the fair $k$-means problem admits a $2c$-approximation in polynomial time.
\end{theorem}

\begin{proof}
Let
\[
g(C)=\frac{\Delta(C, \mathcal{U}_C\cap A)}{|A|}+ \frac{\Delta(C, \mathcal{U}_C\cap B)}{|B|}.
\]
This is basically the $k$-means objective when we consider a weight of $\frac{1}{|A|}$ for the points in $A$ and a weight of $\frac{1}{|B|}$ for the points in $B$.

Let $O$ be an optimal solution to $g$ and $S$ be a $c$-approximation solution to $g$ (i.e., $g(S)\leq cg(O)$). Moreover
\begin{align*}
\Phi(S,\mathcal{U}_S) & = \max\{\frac{\Delta(S, \mathcal{U}_S\cap A)}{|A|}, \frac{\Delta(S, \mathcal{U}_S\cap B)}{|B|}\} \\ & \leq \frac{\Delta(S, \mathcal{U}_S\cap A)}{|A|}+ \frac{\Delta(S, \mathcal{U}_S\cap B)}{|B|} \\ & = g(S).
\end{align*}
Hence $\Phi(S,\mathcal{U}_S)\leq cg(O)$. Now let $O'$ be an optimal solution for $\Phi$. Then
\begin{align*}
g(O') & =\frac{\Delta(O', \mathcal{U}_{O'}\cap A)}{|A|}+ \frac{\Delta(O', \mathcal{U}_{O'}\cap B)}{|B|} \\ & \leq 2\max\{\frac{\Delta(O', \mathcal{U}_{O'}\cap A)}{|A|}, \frac{\Delta(O', \mathcal{U}_{O'}\cap B)}{|B|}\} \\ & = 2\Phi(O',\mathcal{U}_{O'})\leq 2\Phi(S,\mathcal{U}_{S}).
\end{align*}
Also by optimality of $O$ for $g$, we have $g(O)\leq g(O')$. Therefore
\[
g(O)\leq 2\Phi(O',\mathcal{U}_{O'})\leq 2\Phi(S,\mathcal{U}_S) \leq 2cg(O).
\]
This implies that $\frac{\Phi(S,\mathcal{U}_S)}{\Phi(O',\mathcal{U}_{O'})}\leq 2c$.
\end{proof}

\section {Generalization to \texorpdfstring{$m>2$}{m>2} groups}
\label{sec:generalization}

Let $U=A_1\cup \cdots\cup A_m$. Then the objective of fair $k$-means for $m$ demographic groups is to find a set of centers $C$ that minimizes the following
\[
\Phi(C, \mathcal{U}_C) := \max\{\frac{\Delta(C, \mathcal{U}_C\cap A_1)}{|A_1|},\ldots,\frac{\Delta(C, \mathcal{U}_C\cap A_m)}{|A_m|}\},
\]
Let $\mathcal{U}=\{U_1,\ldots,U_k\}$ be a partition of $U$, and $\mu_i^j$ be the mean of $U_i \cap A_j$ (i.e., the mean of members of subgroup $j$ in cluster $i$). Then by a similar argument to Lemma~\ref{lem:lineCenters}, one can conclude that for a fair set of centers $C=\{c_1,\ldots,c_k\}$ with respect to $\mathcal{U}$, $c_i$ is in the convex hull of $\{\mu_i^1,\ldots,\mu_i^m\}$. Then we can generalize the convex program in Equation~\ref{eq:minCenterfirstline} to $m$ demographic groups as the following:
\begin{align}
\label{eq:mgroupsprogram}
\min ~ ~ & \theta\\ \nonumber
\text{s.t. ~} & \frac{\Delta(M^j,\mathcal{U}\cap A_j)}{|A_j|} +
\sum_{i \in [k]} \alpha_i^j \|c_i-\mu_i^j\|^2 \leq \theta \hspace{4mm},  j \in [m] \\ \nonumber
& c_i \in Conv(\mu_i^1,\ldots,\mu_i^m) \hspace{4mm},  i \in [k]
\end{align}
where $\alpha_i^j=\frac{|U_i\cap A_j|}{|A_j|}$ and $M^j=\{\mu_1^j,\ldots,\mu_k^j\}$. The set of $c_i$'s found by solving the above convex program will be a fair set of centers with respect to $\mathcal{U}$. We can solve this using standard convex optimization algorithms including gradient descent. However, similar to the case of two groups, we can find a fair set of centers by searching a standard $(m-1)$-simplex. Namely, we only need to search the following set to find a fair set of centers.
\[
Z=\{C=(c_1,\ldots,c_k): c_i = \sum_{j=1}^m \frac{\gamma_j \alpha_i^j}{\gamma_1 \alpha_i^1+\cdots+\gamma_m \alpha_i^m}\mu_i^j ~ ~ , \sum_{j=1}^m \gamma_j = 1\}
\]

The following notations will be convenient. For $C=(c_1\ldots,c_k)$, and $j\in [m]$, let \[f_j(C)=\frac{\Delta(M^j,\mathcal{U}\cap A_j)}{|A_j|} +
\sum_{i \in [k]} \alpha_i^j \|c_i-\mu_i^j\|^2,\]
and $F(C)=\max_{j\in [m]} f_j(C)$. Then the convex program represented in (\ref{eq:mgroupsprogram}) is equivalent to $\min F(C): c_i\in Conv(\mu_i^1,\ldots,\mu_i^m) ~ \forall i\in [k]$.

Let $u_i^j = c_i-\mu_i^j$. For a vector $v$, let $v(s)$ denote its $s$'th component. Then we have
\[
\|c_i-\mu_i^j\|^2=\sum_{s=1}^d u_i^j(s)^2
\]

\begin{theorem}
\label{thm:mgroupsZ}
Any optimum solution of (\ref{eq:mgroupsprogram}) is in $Z$.
\end{theorem}
\begin{proof}
We can view a set of centers as a point in a $k\times d$ dimensional space. Let $\{e_{i,s}: i\in [k], s\in[d]\}$ be the set of standard basis of this space. Then we have
\[
\frac{d}{d e_{i,s}} f_j(C) = 2\alpha_i^j u_i^j(s).
\]
By Fact \ref{fact:0subgrad} and \ref{fact:maxsubgrad}, we only need to show that set $Z$ is the set of all points for which there exists a convex combinations of $\nabla f_1(C) , \ldots ,\allowbreak \nabla f_m(C)$ that is equal to $\vec{0}$.
Let $0\leq \gamma_1,\ldots,\gamma_m\leq 1$ such that $\sum_{t=1}^m \gamma_t=1$. We want to find a $C$ such that $\sum_{j=1}^m \gamma_j \nabla f_j(C) = \vec{0}$. Therefore for each $i,s$, we have
\[
0 = \sum_{j=1}^m \gamma_j \frac{d}{d e_{i,s}} f_j(C) = \sum_{j=1}^m \gamma_j (2\alpha_i^j u_i^j(s)) = \sum_{j=1}^m 2 \gamma_j \alpha_i^j (c_i(s)-\mu_i^j(s)).
\]
Thus
\[
c_i(s) = \sum_{j=1}^m \frac{\gamma_j \alpha_i^j}{\gamma_1 \alpha_i^1+\cdots+\gamma_m \alpha_i^m}\mu_i^j(s),
\]
and
\[
c_i = \sum_{j=1}^m \frac{\gamma_j \alpha_i^j}{\gamma_1 \alpha_i^1+\cdots+\gamma_m \alpha_i^m}\mu_i^j
\]
This shows that the set of centers that satisfy $\sum_{j=1}^m \gamma_j \nabla f_j(C) = 0$, for some $\gamma_1,\ldots,\gamma_m$, are exactly the members of $Z$.
\end{proof}
\vspace{-1mm}
Note that any element in $Z$ is identified by a point in the standard $(m-1)$-simplex, i.e., $(\gamma_1,\ldots,\gamma_m)$ such that $\sum_{j=1}^m \gamma_j = 1$. However, the function defined on the $(m-1)$-simplex is not necessarily convex. Indeed, as we will show in Figure~\ref{fig:qcexample} in the Appendix, it is not even quasiconvex. Thus, one can either use standard convex optimization algorithms to solve the original convex program in ($\ref{eq:mgroupsprogram}$), or other heuristics to only search the set $Z$. For our experiments, we use a variant of the multiplicative weight update algorithm on set $Z$ --- see Algorithm~\ref{alg:MWU}. 

To certify the optimality of the solution, one can use Fact~\ref{fact:0subgrad} and show that $\vec{0}$ is a subgradient. However, the iterative algorithms usually do not find the exact optimum, but rather converge to the optimum solution. To evaluate the distance of a solution from the optimum, we propose a min/max theorem for set $Z$ in Section~\ref{sec:certificate}. This theorem allows us to certify that the solutions found by our heuristic in the experiments are within a distance of 0.01 from the optimal.

\subsection{Multiplicative Weight Update Heuristic}
Note that the original optimization problem given in (\ref{eq:mgroupsprogram}) is convex. However we can use a heuristic to solve the problem in the $\gamma$ space. One such heuristic is the multiplicative weight update algorithm \citep{arora2012multiplicative}, precisely defined as Algorithm~\ref{alg:MWU}.

\vspace{-1mm}
\RestyleAlgo{algoruled}
\begin{algorithm}[ht!]
\footnotesize
\textbf{Input:} Integers $m$ and $k$, numbers $\alpha_i^j$, and centers $\mu_i^j$ for $i\in[k]$ and $j\in [m]$, and $\frac{\Delta(M^j,\mathcal{U}\cap A_j)}{|A_j|}$ for $j\in [m]$.\\[0.6ex]

$\gamma_j\leftarrow \frac{1}{m}$, for $j\in [m]$

\For{$t = 1,\ldots, T$}{
	$c_i \leftarrow \sum_{j=1}^m \frac{\gamma_j \alpha_i^j}{\gamma_1 \alpha_i^1+\cdots+\gamma_m\alpha_i^m} \mu_i^j$, for $i\in [k]$
	
	$C\leftarrow (c_1,\ldots,c_k)$
	
	Compute $f_j(C)$ for all $j\in [m]$

    $F(C) \leftarrow \max_{j\in[m]} f_j(C)$
    
    $d_j \leftarrow F(C) - f_j(C)$, for $j\in [m]$
    
    $\gamma_j \leftarrow \gamma_j(1-\frac{d_j}{\sqrt{t+1}\max_{j\in[m]} d_j})$

    Normalize $\gamma_j$'s such that $\sum_{j=1}^m \gamma_j = 1$
}

\Return $C$
\caption{Multiplicative Weight Update}
\label{alg:MWU}
\end{algorithm}
\vspace{-1mm}

\subsection{Certificate of Optimality}
\label{sec:certificate}
Next, we give a min/max theorem that can be used to find a lower bound for the optimum value. Using this theorem, we can certify that, in practice, the multiplicative weight update algorithm finds a solution very close to the optimum.

\begin{theorem}
\label{thm:minmax}
Let $S\subseteq [m]$ and 
\begin{align*}
Z_S= & \{C=(c_1,\ldots,c_k): c_i = \small\sum_{j=1}^m \frac{\gamma_j \alpha_i^j}{\gamma_1 \alpha_i^1+\cdots+\gamma_m \alpha_i^m}\mu_i^j ~ ~ , \\ & \sum_{j=1}^m \gamma_j = 1, \text{ and } \gamma_j=0, \forall j\notin S\}.
\end{align*}
Then \[\max_{C\in Z_S} \min_{j\in S} f_j(C)\leq \min_{C\in Z_S} \max_{j\in S} f_j(C).\] 
Moreover \[\max_{C\in Z_S} \min_{j\in S} f_j(C)\leq \min_{C\in Z_{[m]}} \max_{j\in [m]} f_j(C).\]
\end{theorem}
\begin{proof}
Let $C,C'\in Z_S$ and let $\gamma,\gamma'$ be the corresponding parameters for $C,C'$, respectively. Note that $Z_S\subseteq Z$. Therefore $\sum_{j\in[m]} \gamma_j \nabla f_j(C)=\vec{0}$. Hence because $\gamma_j=0$ for any $j\notin S$, we have
$\sum_{j\in S} \gamma_j \nabla f_j(C)=\vec{0}$. Hence $(\sum_{j\in S} \gamma_j \nabla f_j(C)) \cdot (C'-C)=0$. Therefore there exists a $j^*\in S$ such that $\nabla f_{j^*}(C) \cdot (C'-C) \geq 0$. Thus because $f_{j^*}$ is convex, we have
\[
f_{j^*}(C') \geq f_{j^*}(C) + \nabla f_{j^*}(C) \cdot (C'-C) \geq f_{j^*}(C)
\]
Therefore $\max_{j\in S} f_j(C') \geq \min_{j\in S} f_j(C)$. Note that this holds for any $C,C'\in Z$ and this implies the first part of the theorem.

Let $C'$ be the optimum solution to $\min_{C\in Z_S} \max_{j\in S} f_j(C)$. Note that by Theorem~\ref{thm:mgroupsZ}, this is an optimum solution to the problem of finding a fair set of centers for the groups in $S$. Moreover for any set of centers $C''$ outside the convex hull of the centers of groups in $S$, we have $\max_{j\in S} f_j(C')\leq \max_{j\in S} f_j(C'')$ --- proof of this is similar to Lemma~\ref{lem:lineCenters}. Hence $\max_{j\in S} f_j(C')\leq \min_{C\in Z_{[m]}} \max_{j\in S} f_j(C)$. Thus we have
\begin{align*}
\max_{C\in Z_S} \min_{j\in S} f_j(C) & \leq \min_{C\in Z_S} \max_{j\in S} f_j(C)=\max_{j\in S} f_j(C') \\ & \leq \min_{C\in Z_{[m]}} \max_{j\in S} f_j(C) \leq \min_{C\in Z_{[m]}} \max_{j\in [m]} f_j(C)
\end{align*}
\end{proof}
\vspace{-2mm}
Note that we can use Theorem~\ref{thm:minmax}, to get a lower bound on the optimum solution of the convex program in (\ref{eq:mgroupsprogram}). For example, suppose $C'$ is a solution returned by a heuristic. Then 
\[
\min_{j\in [m]} f_j(C') \leq \max_{C\in Z_{[m]}} \min_{j\in [m]} f_j(C) \leq \min_{C\in Z_{[m]}} \max_{j\in [m]} f_j(C).
\]
Therefore $\min_{j\in [m]} f_j(C')$ is a lower bound for the optimum solution. Hence the difference of the solution returned by the heuristic with the optimum solution is at most \[(\max_{j\in [m]} f_j(C'))-(\min_{j\in [m]} f_j(C')).\] This will be very useful for the case where $f_j(C^*)=F(C^*)$ for all $j\in [m]$, where $C^*$ is the optimum solution. The reason is that in this case, Theorem~\ref{thm:minmax} implies $\max_{C\in Z_{[m]}} \min_{j\in [m]} f_j(C) = \min_{C\in Z_[m]} \max_{j\in [m]} f_j(C)$. However this might not be the case and we might have $f_j(C^*)<F(C^*)$ for some $j$. In this case we can use $S\subset [m]$. For example an $S$ that gives a larger lower bound and for which $\max_{C\in Z_S} \min_{j\in S} f_j(C) = \min_{C\in Z_S} \max_{j\in S} f_j(C)$.

\subsection{Stability and Approximability}
We conclude this section by a discussion on the stability and the approximability of fair $k$-means for $m$ groups.

Our stability results generalizes to $m$ demographic groups. 
Let $C^*=\{c_1^*,\ldots,c_k^*\}$ be an optimal solution, and $S\subseteq [m]$. Also let $f_j(C^*) = \max_{i\in [m]} f_i(C^*)$ for $j\in S$, and $f_j(C^*) < \max_{i\in [m]} f_i(C^*)$ for $j\notin S$. Then one can see that, for all $i\in [k]$, $c_i^*\in Conv(\{\mu_i^j:j\in S\})$. This uniquely determines the location of the optimal solution, and thus we cannot improve the value of functions $f_j$ where $j\notin S$. Moreover, with an argument similar to Proposition~\ref{prop:stability}, one can deduce that we cannot improve the value of functions $f_j$ where $j\in S$.

Moreover, the Fair-Lloyd algorithm for $m$ demographic groups converges to a solution in finite time, essentially because the number of possible partitions of points is finite. Finally, if the $k$-means problem admits a $c$-approximation then the fair $k$-means problem for $m$ demographic groups admits an $mc$-approximation --- the proof is similar to the proof of Theorem~\ref{thm:approx}. 

\begin{figure*}[t]
\rotatebox[origin=c]{90}{w/o PCA}
\centering
\begin{subfigure}{.31\textwidth}
  \centering
  \caption*{LFW dataset}
  \includegraphics[width=\linewidth]{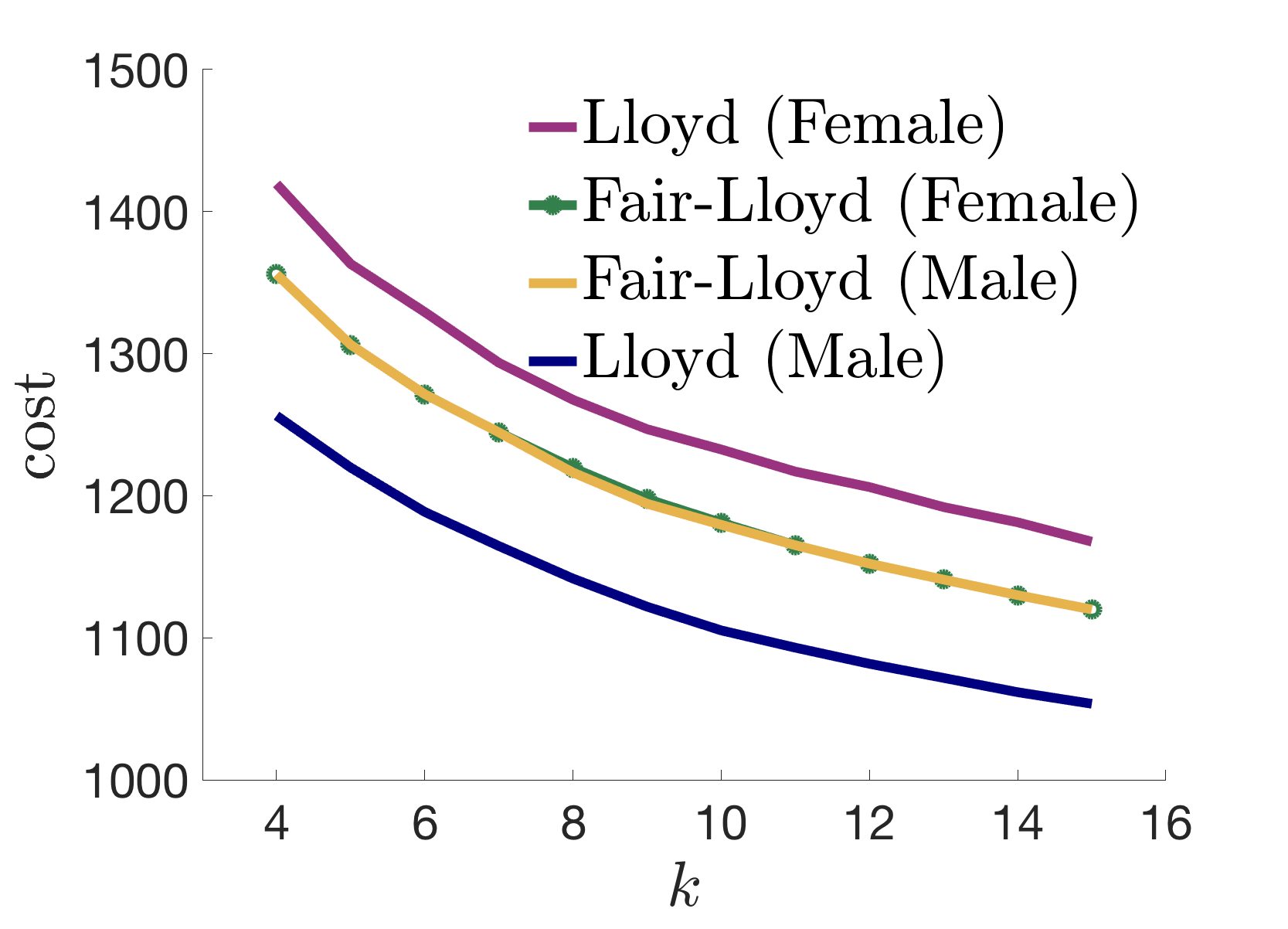}
\end{subfigure}
\begin{subfigure}{.31\textwidth}
  \centering
  \caption*{Adult dataset}
  \includegraphics[width=\linewidth]{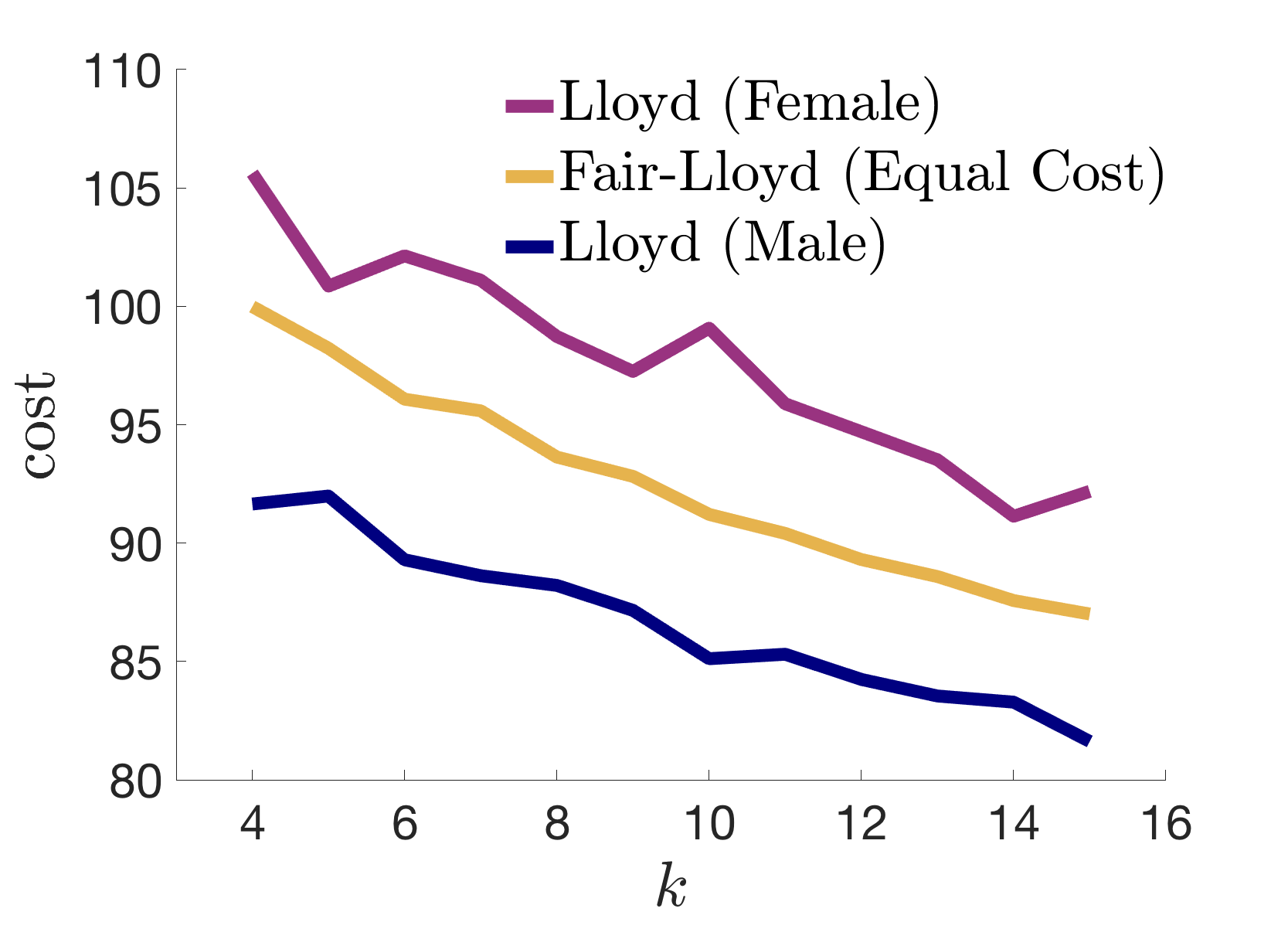}  
\end{subfigure}
\begin{subfigure}{.31\textwidth}
  \centering
  \caption*{Credit dataset}
  \includegraphics[width=\linewidth]{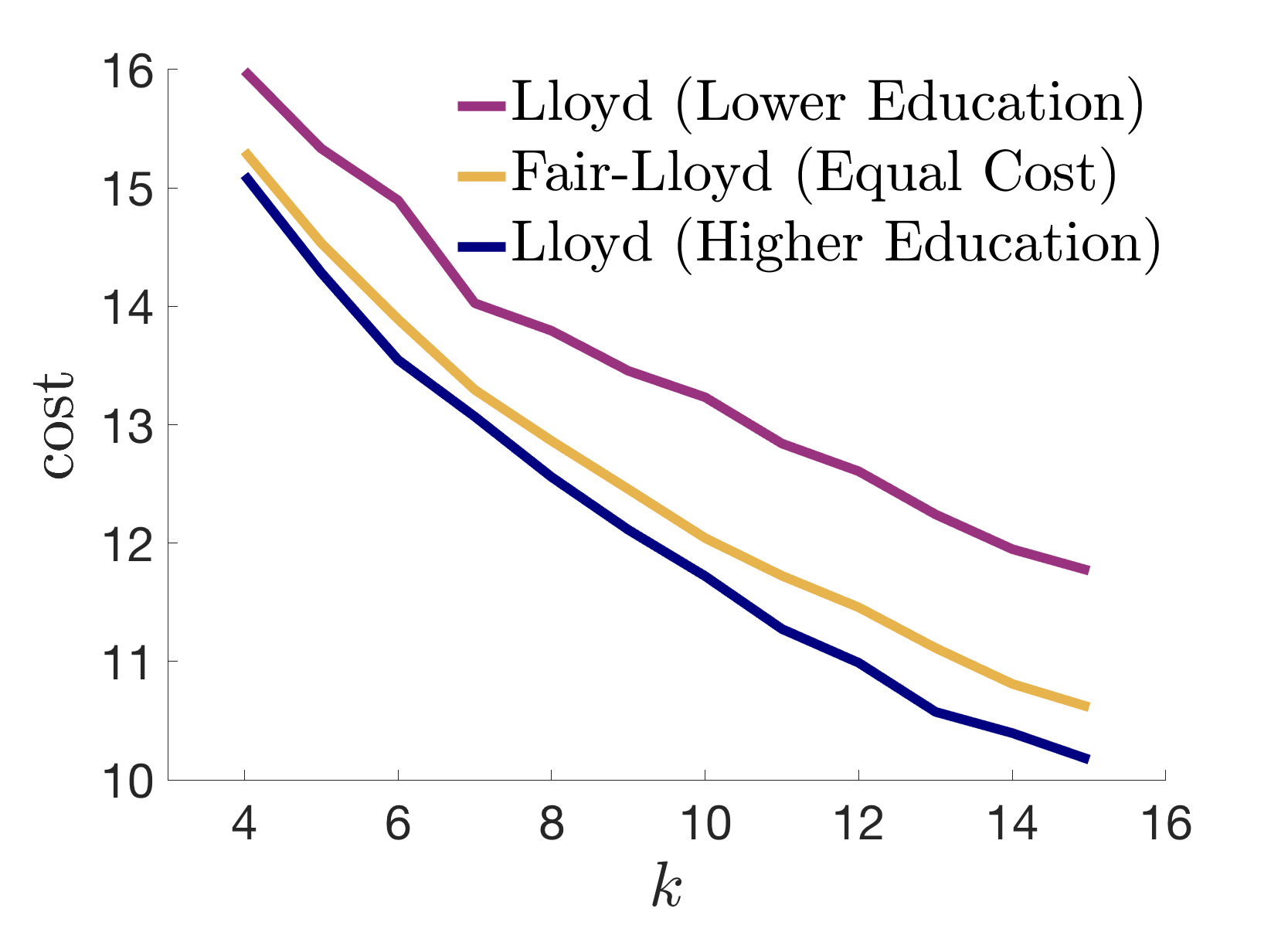}  
\end{subfigure}
\hfill

\rotatebox[origin=c]{90}{w/ PCA}
\centering
\begin{subfigure}{.31\textwidth}
  \centering
  \includegraphics[width=\linewidth]{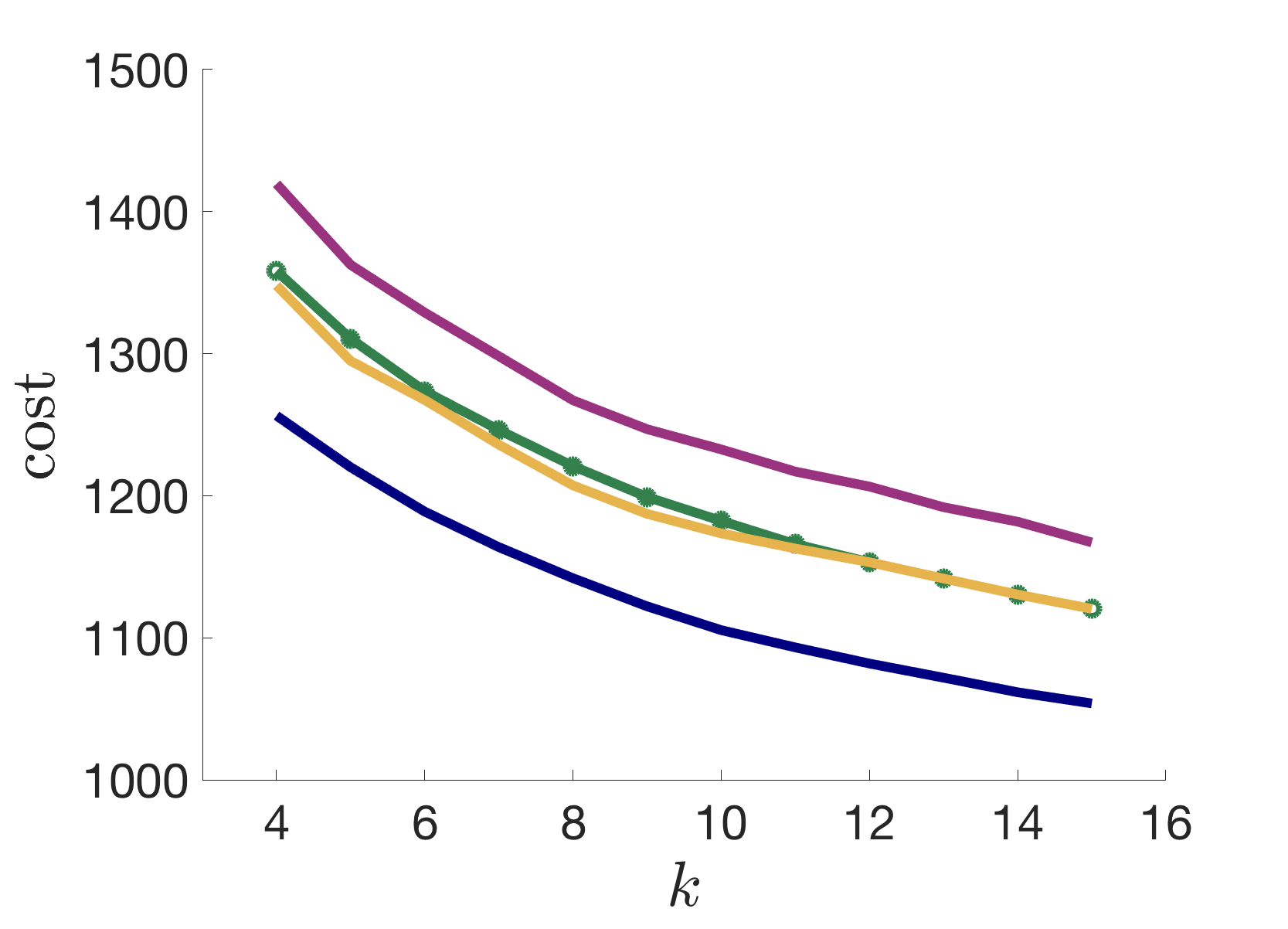}
\end{subfigure}
\begin{subfigure}{.31\textwidth}
  \centering
  \includegraphics[width=\linewidth]{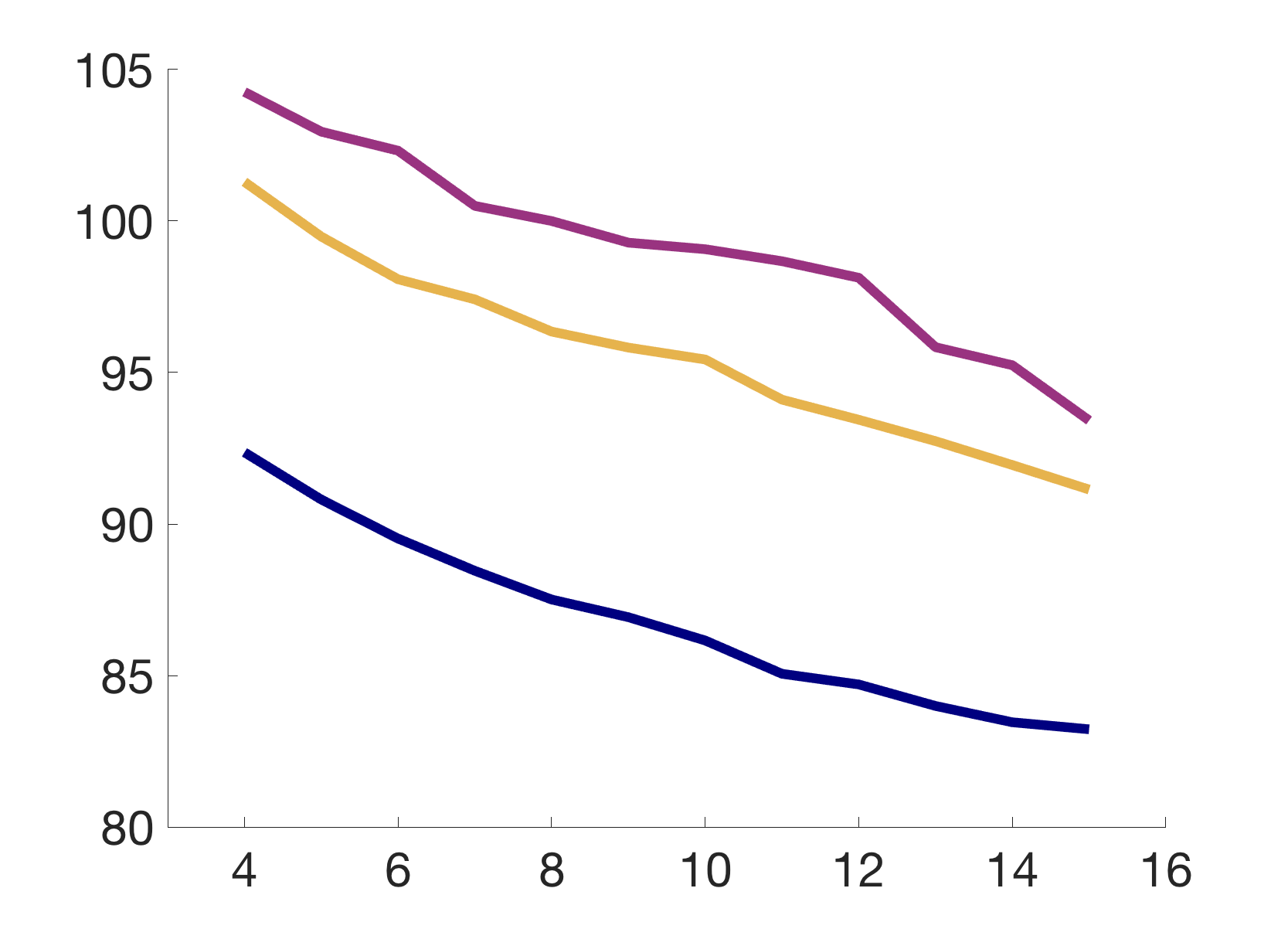}  
\end{subfigure}
\begin{subfigure}{.31\textwidth}
  \centering
  \includegraphics[width=\linewidth]{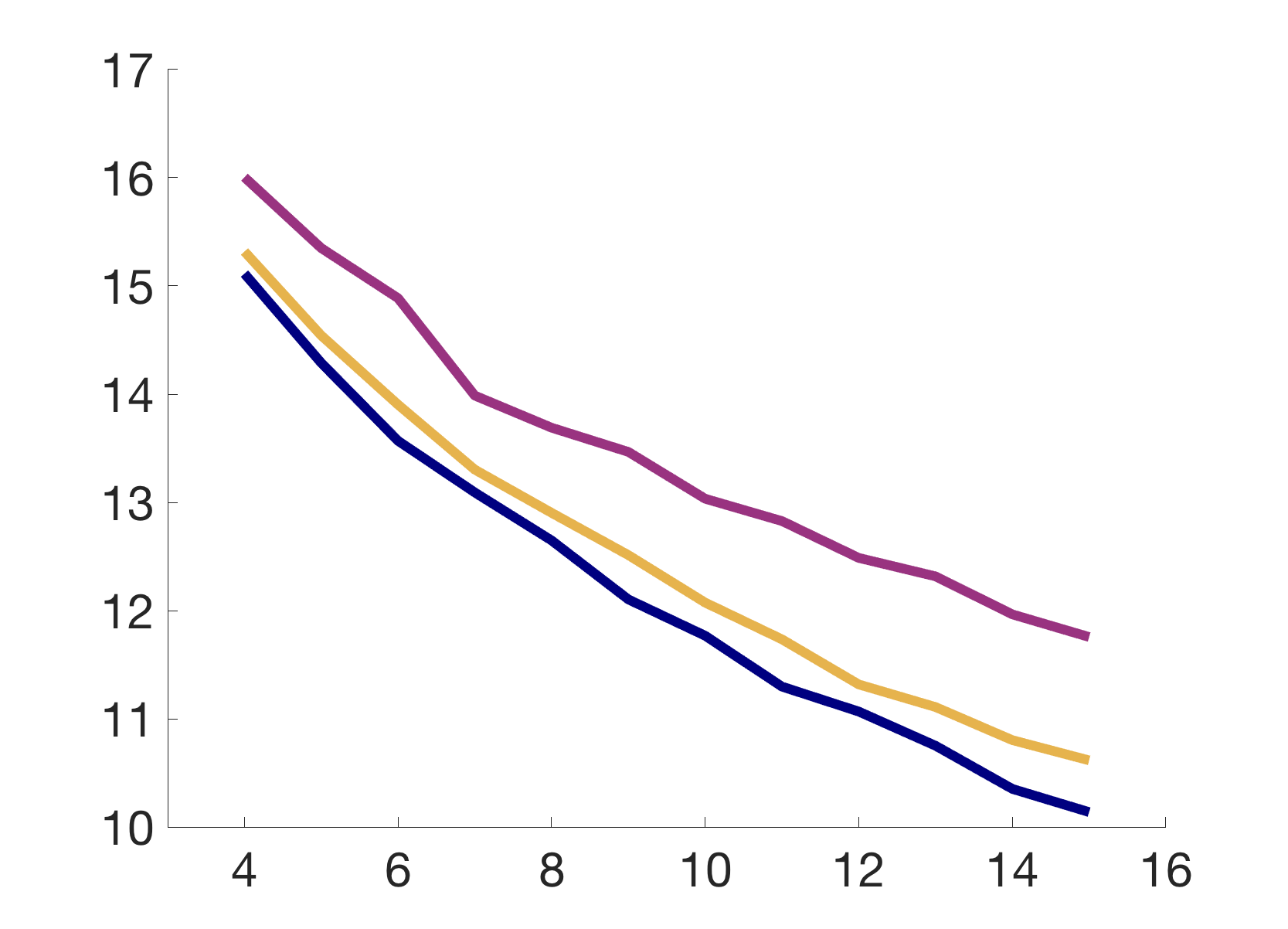}  
\end{subfigure}
\hfill

\rotatebox[origin=c]{90}{w/ Fair-PCA}
\centering
\begin{subfigure}{.31\textwidth}
  \centering
  \includegraphics[width=\linewidth]{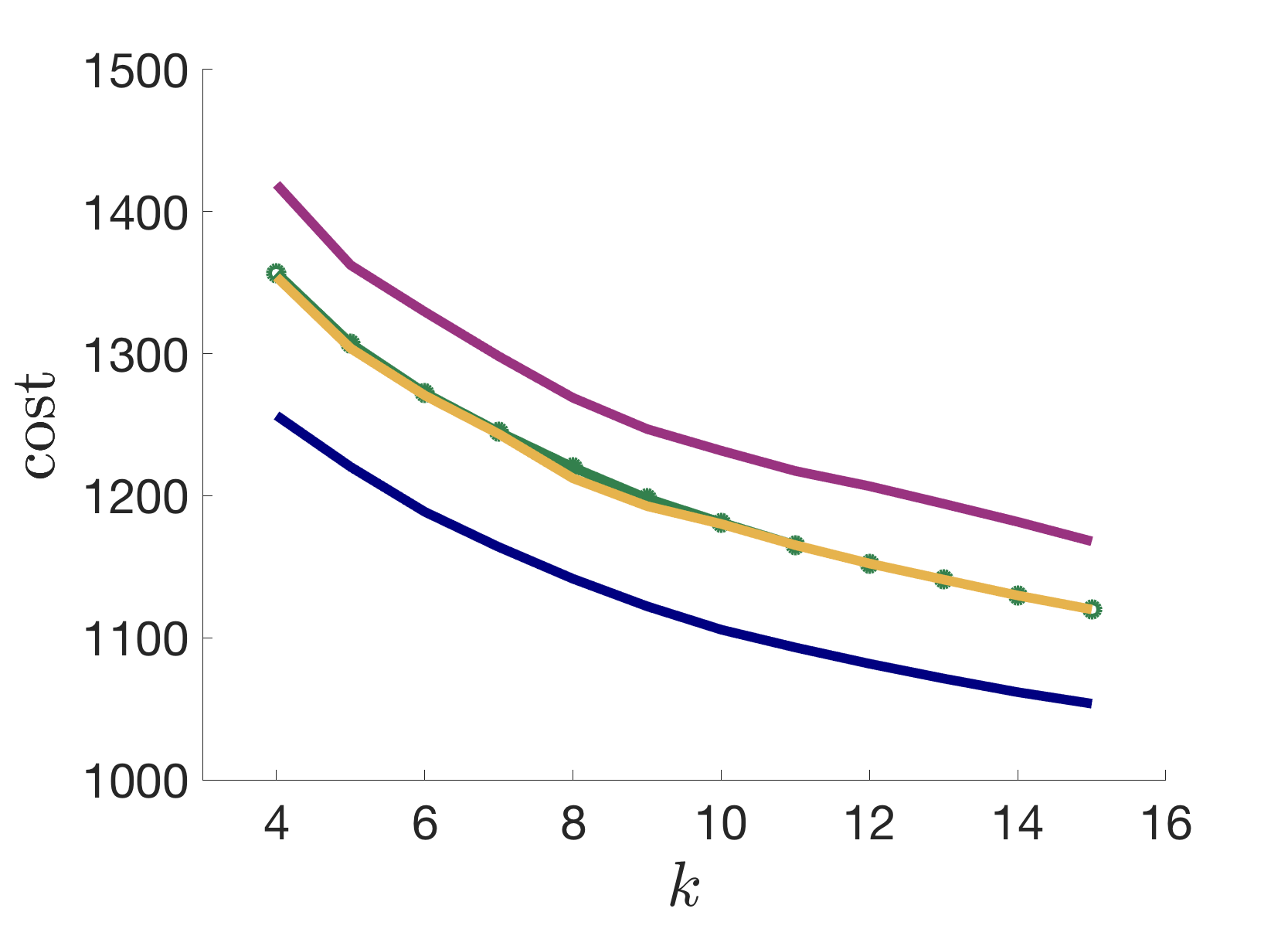}
\end{subfigure}
\begin{subfigure}{.31\textwidth}
  \centering
  \includegraphics[width=\linewidth]{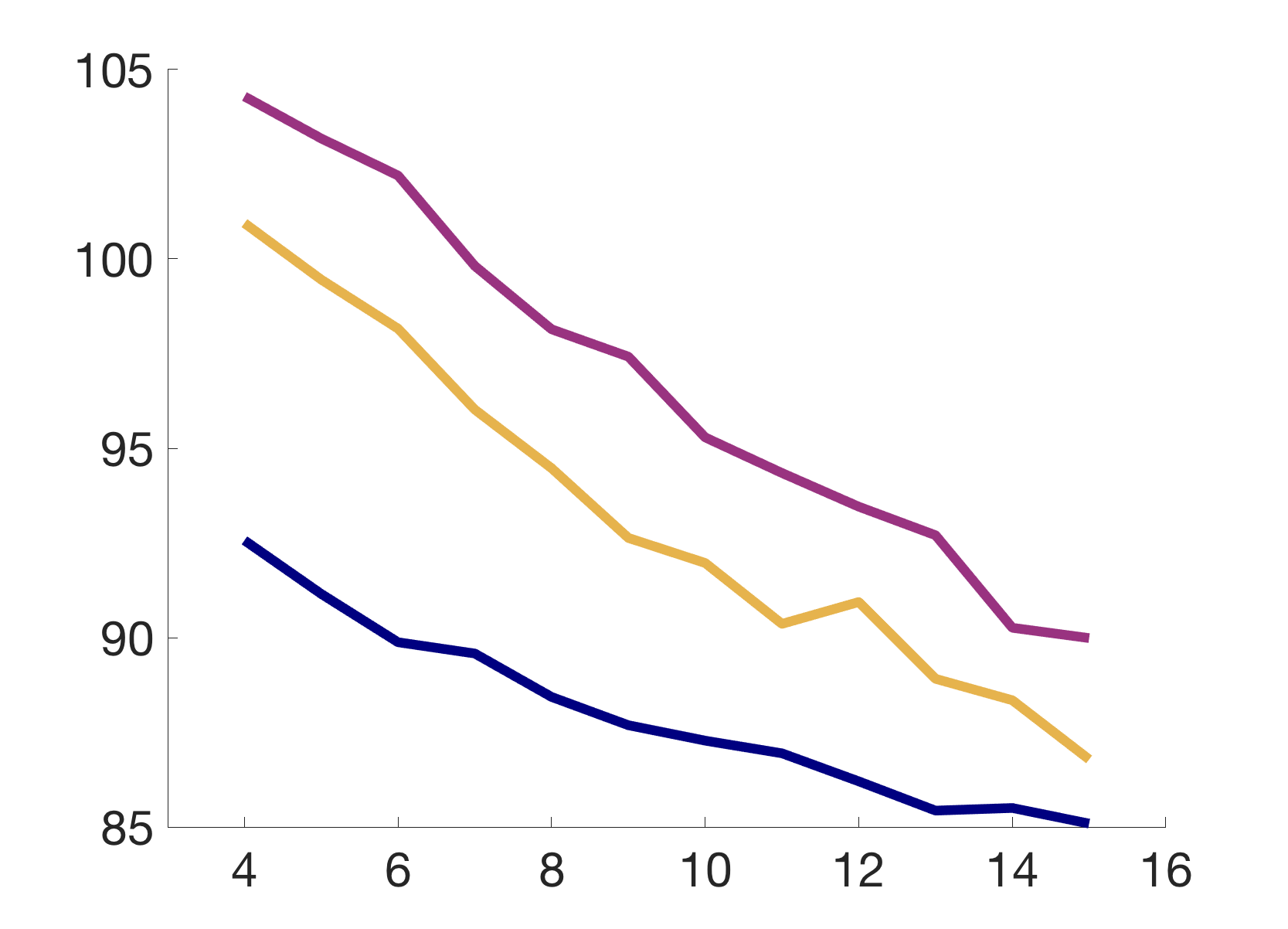}  
\end{subfigure}
\begin{subfigure}{.31\textwidth}
  \centering
  \includegraphics[width=\linewidth]{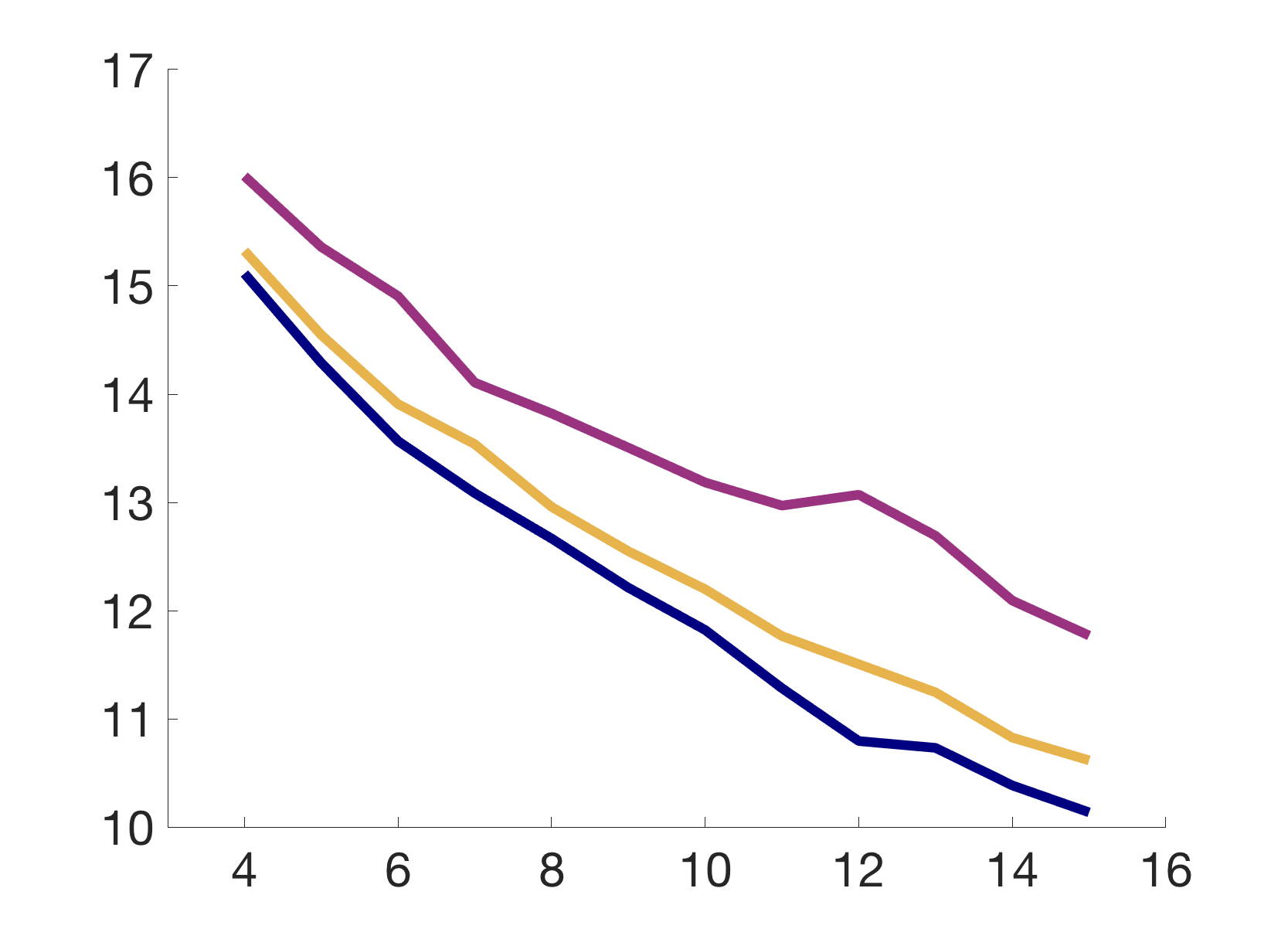}  
\end{subfigure}
\caption{Average clustering cost of different groups when using Fair-Lloyd algorithm versus the standard Lloyd's. Rows correspond to different 
pre-processing
methods and columns to the datasets. Note that the fair clustering costs for the two groups are identical or nearly identical in all datasets.}
\label{fig:cost}
\end{figure*}

\section{Experimental Evaluation}
\label{sec:eval}
We consider a clustering to be fair if it has equal clustering costs across different groups. We compare the average clustering cost for different demographic groups on multiple benchmark datasets, using Lloyd's algorithm and Fair-Lloyd algorithm. The code of our experiments is publicly available at \url{https://github.com/fairkmeans/Fair-K-Means-Clustering}.

\begin{figure}[b]
\centering
    \includegraphics[width=.57\linewidth]{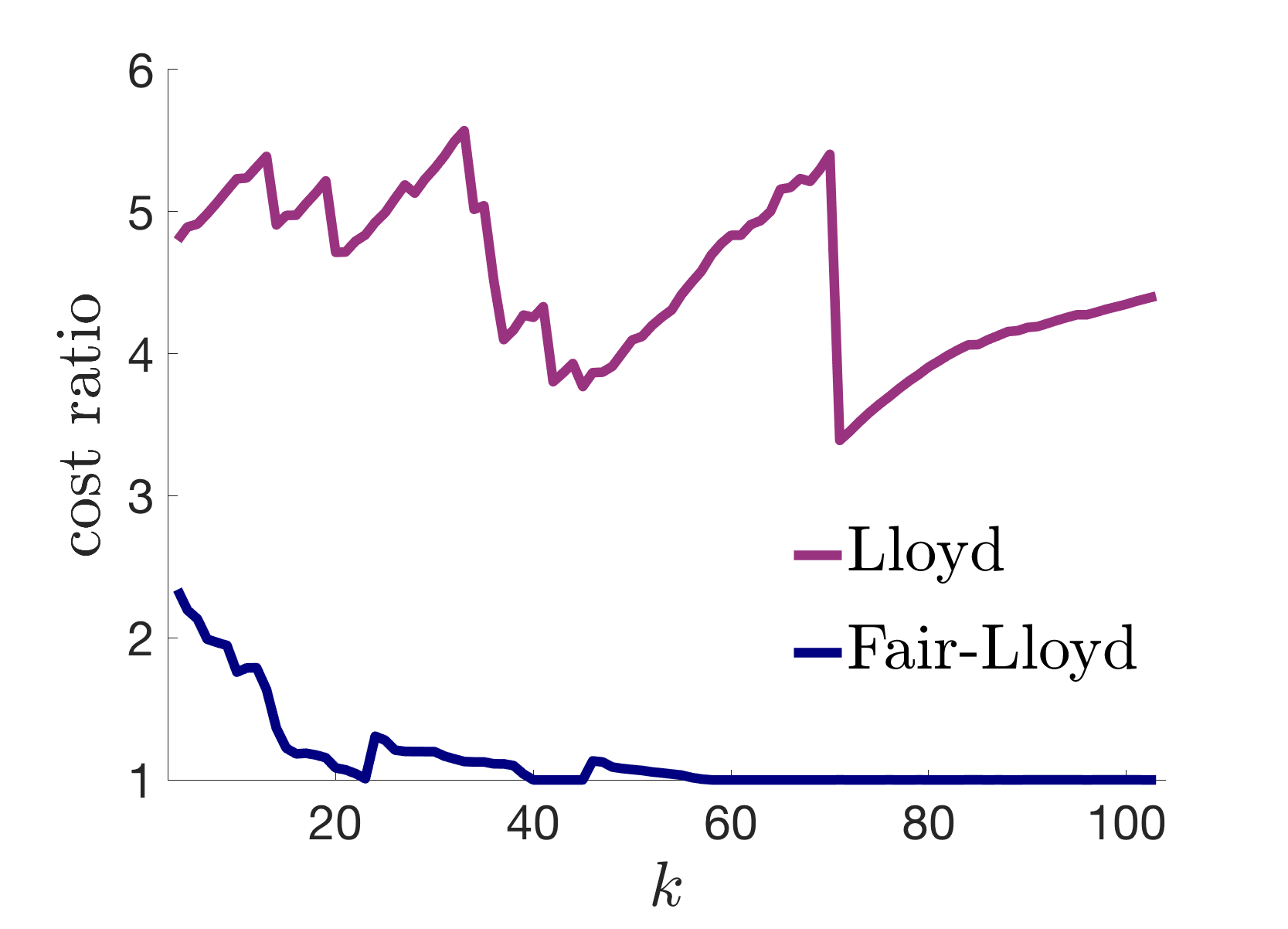}
    \caption{
      Adult dataset: The maximum ratio of average clustering cost between any two racial groups: ``Amer-Indian-Eskim'', ``Asian-Pac-Islander'', ``Black'', ``White'', and ``Other''.
    } \label{fig:adult5RacesRatio}
\end{figure}

\begin{figure}[b]
\begin{subfigure}{.58\columnwidth}
    \centering
    \includegraphics[width=1.1\linewidth]{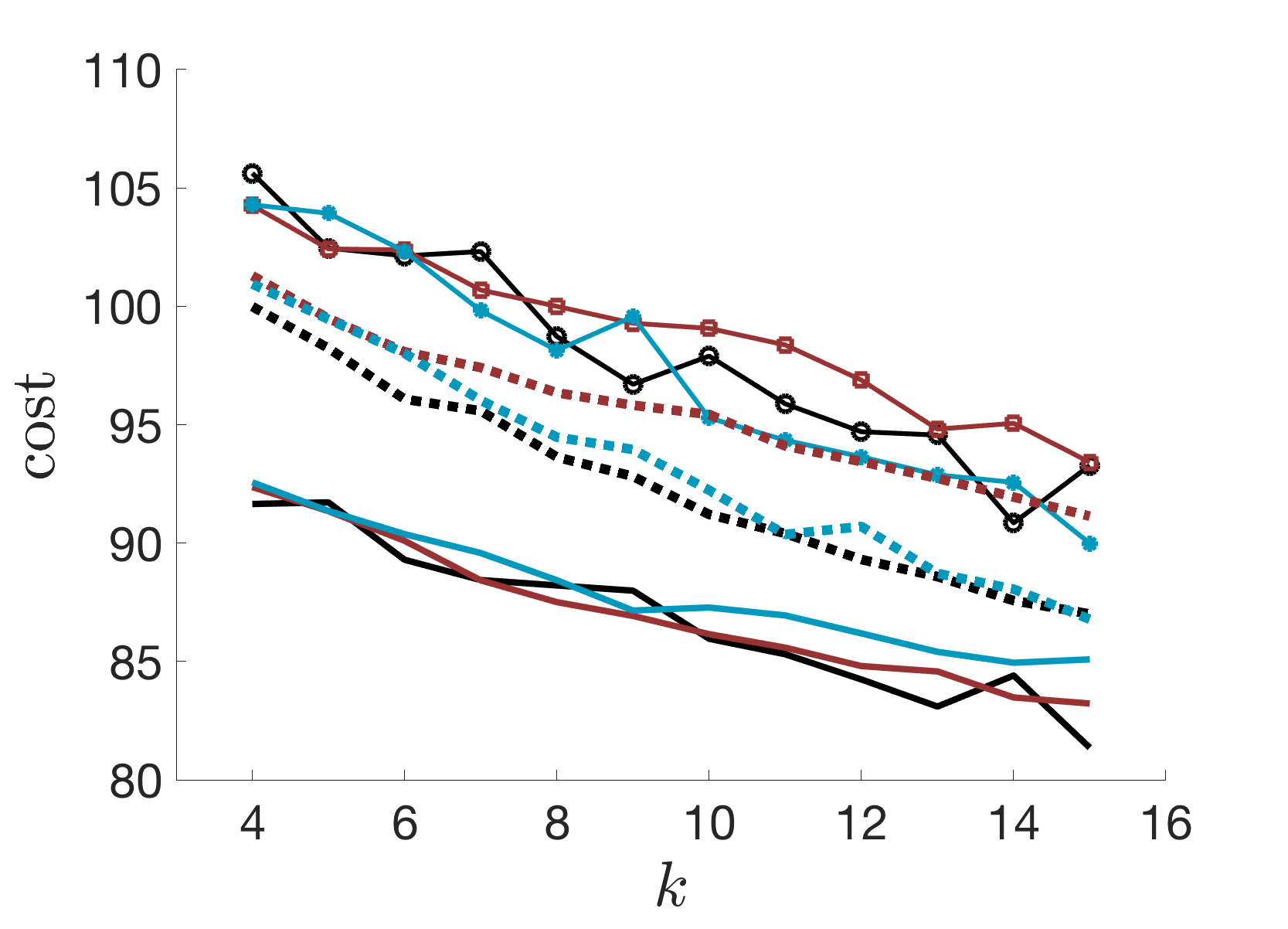}  
\end{subfigure}
\begin{subfigure}{.4\columnwidth}
    \centering
    \includegraphics[width=1\linewidth]{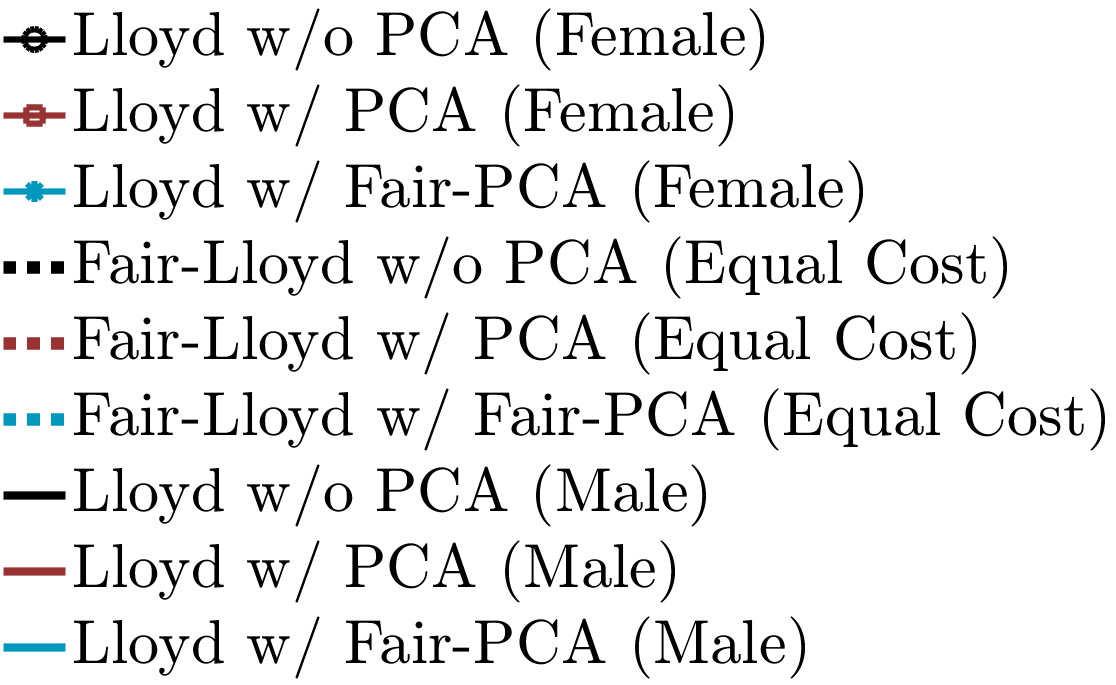}  
\end{subfigure}
\caption{Adult dataset: comparison of the standard Lloyd's and Fair-Lloyd algorithm for the three different 
pre-processing
choices of w/o PCA, w/ PCA, and w/ Fair-PCA. }    
\label{fig:adultAll}
\end{figure}

\begin{figure*}[!t]
\begin{subfigure}{.33\textwidth}
  \centering
  \caption*{LFW dataset}
  \includegraphics[width=\linewidth]{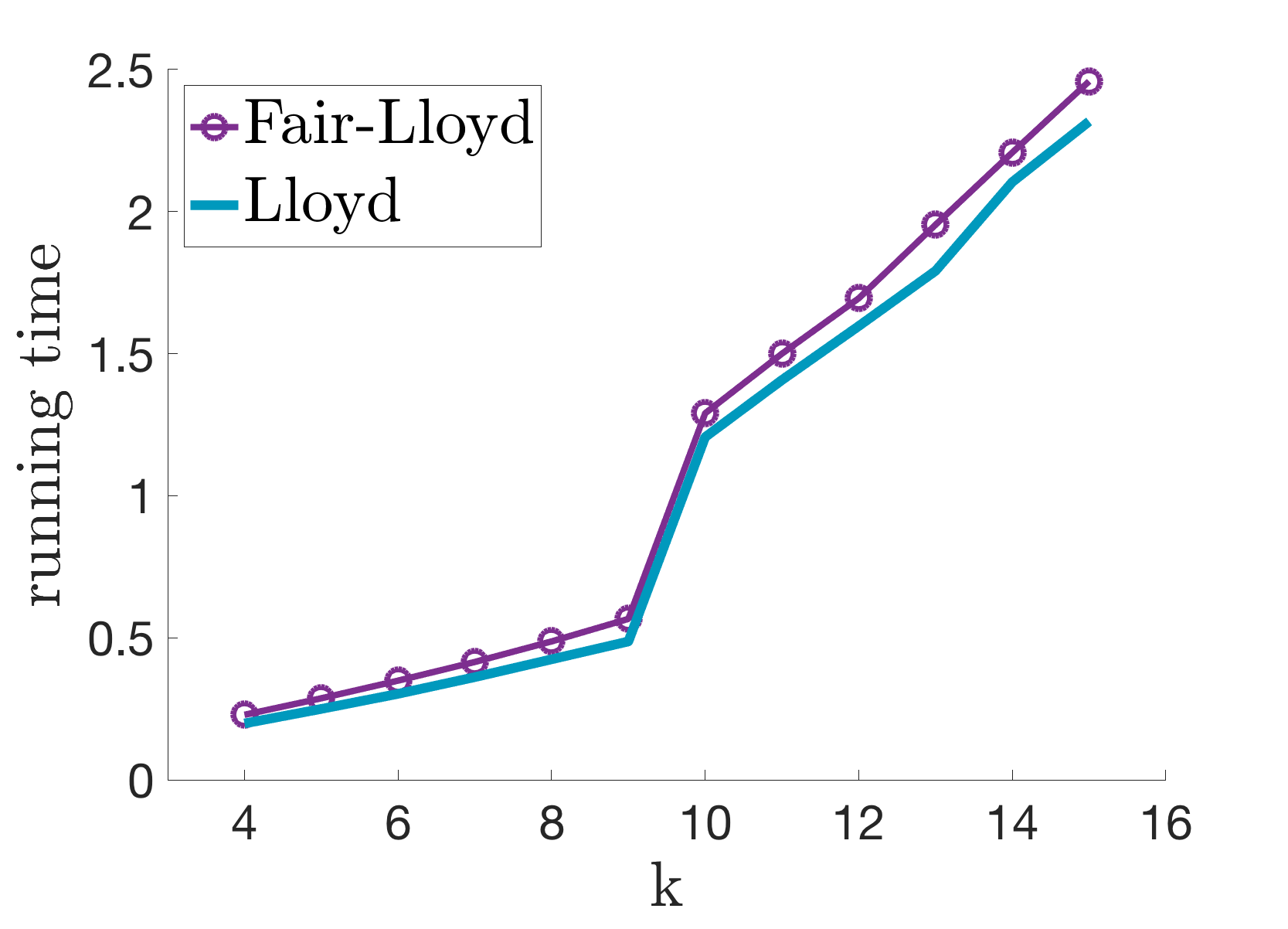}
\end{subfigure}
\begin{subfigure}{.33\textwidth}
  \centering
  \caption*{Adult dataset}
  \includegraphics[width=\linewidth]{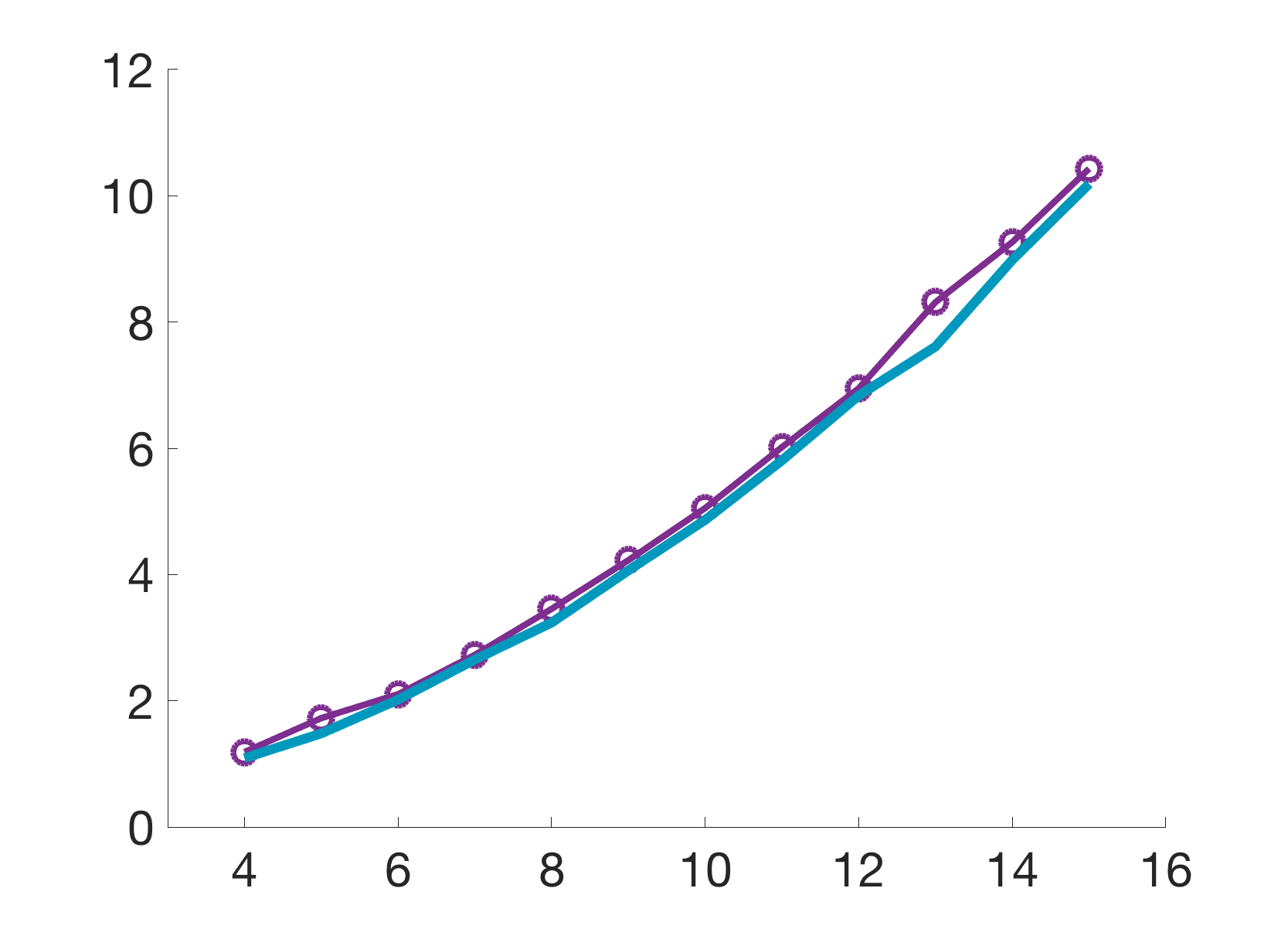} 
\end{subfigure}
\begin{subfigure}{.33\textwidth}
  \centering
  \caption*{Credit dataset}
  \includegraphics[width=\linewidth]{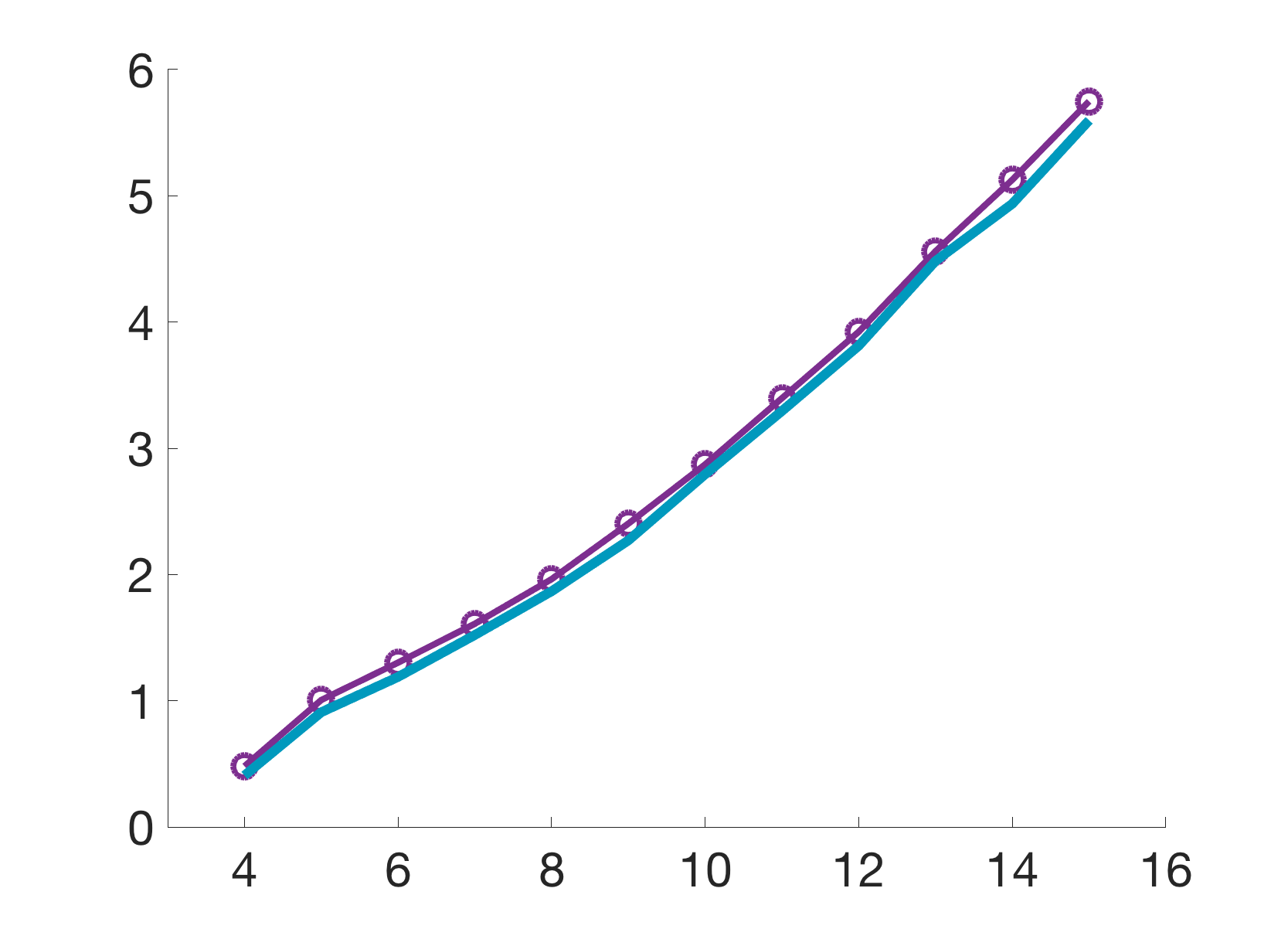}  
\end{subfigure}
\caption{Running time (seconds) of Fair-Lloyd algorithm versus the standard Lloyd's algorithm on the $k$-dimensional PCA space for $200$ iterations.}
\label{fig:runtime}
\end{figure*}

\begin{figure*}[t]
\begin{subfigure}{.33\textwidth}
  \centering
  \caption*{LFW dataset}
  \includegraphics[width=\linewidth]{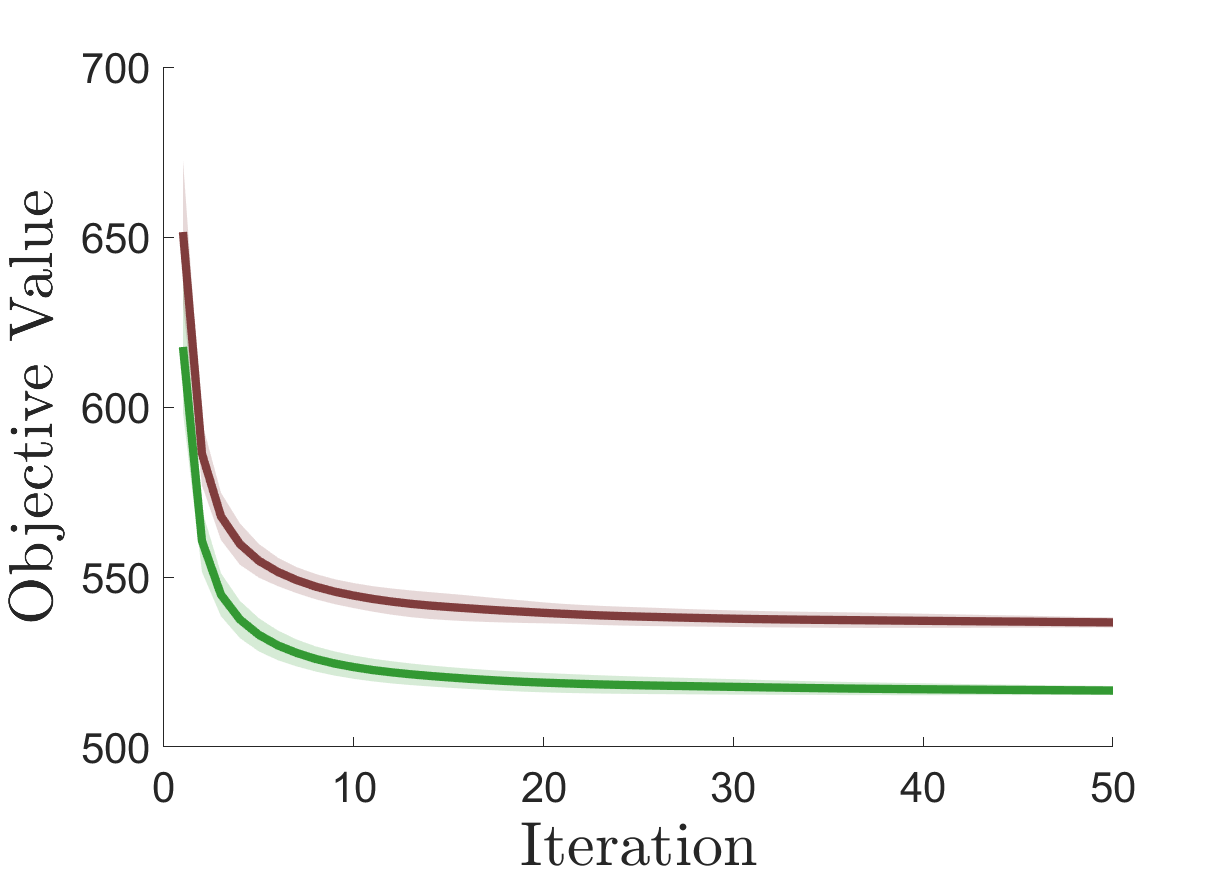}
\end{subfigure}
\begin{subfigure}{.33\textwidth}
  \centering
  \caption*{Adult dataset}
  \includegraphics[width=\linewidth]{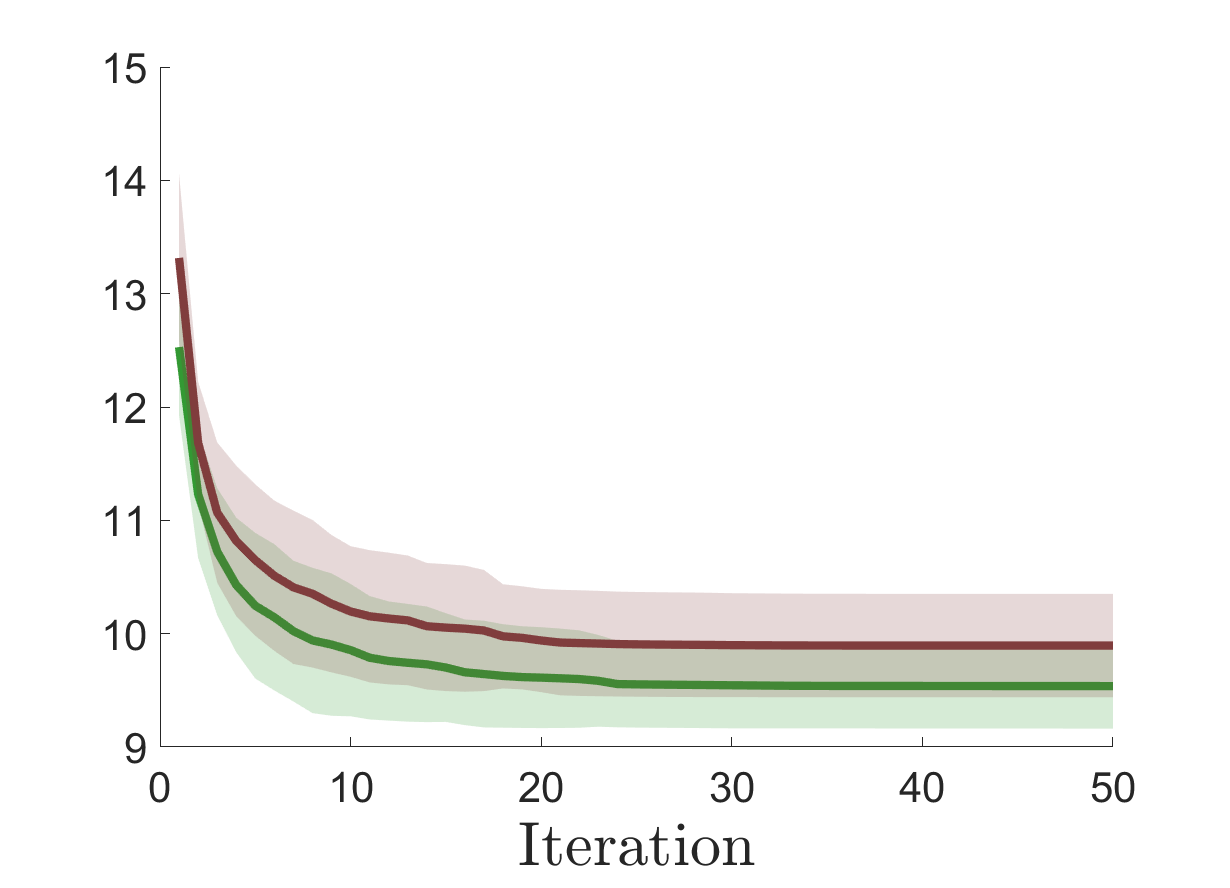} 
\end{subfigure}
\begin{subfigure}{.33\textwidth}
  \centering
  \caption*{Credit dataset}
  \includegraphics[width=\linewidth]{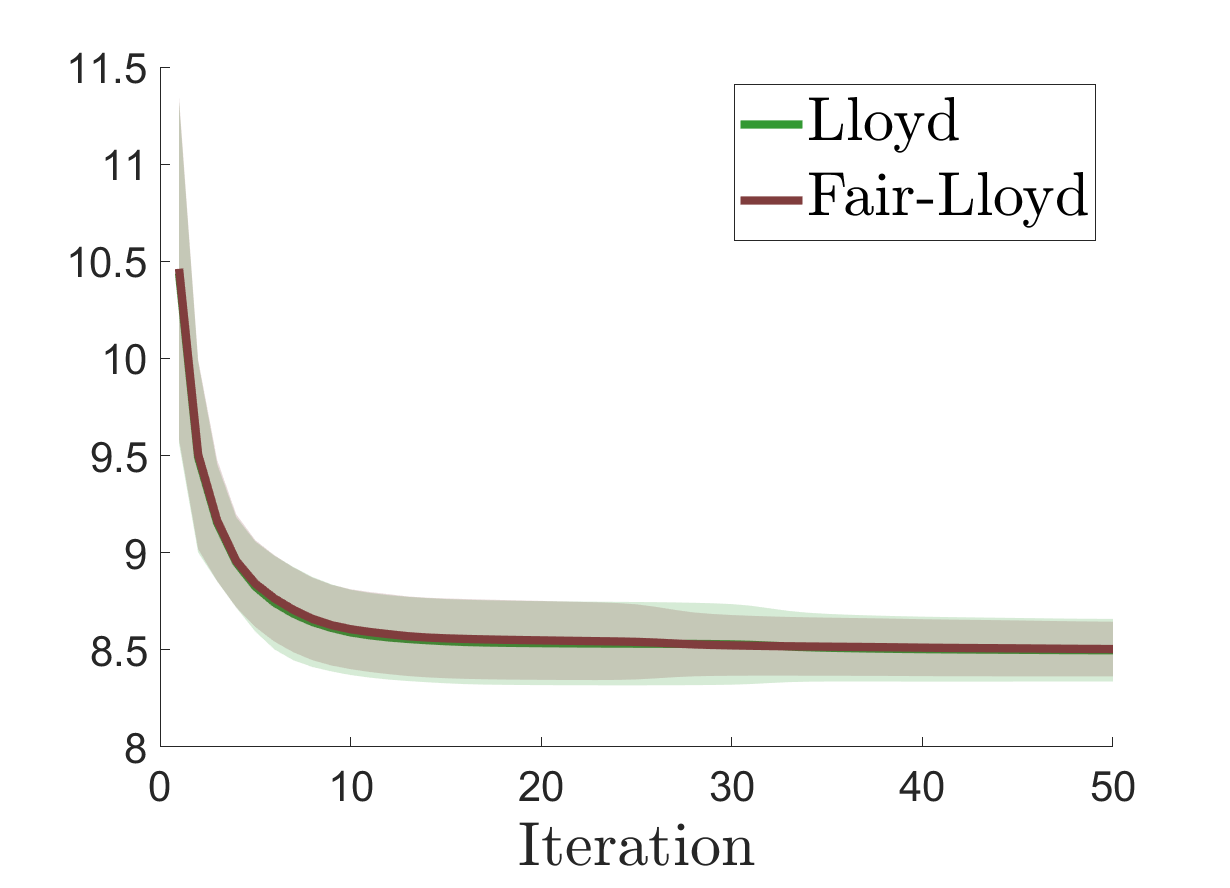}  
\end{subfigure}
\caption{Convergence rate of Fair-Lloyd algorithm versus the standard Lloyd's algorithm for $k=10$. The plotted objective value for the standard Lloyd is the average cost of clustering over the whole population, and the objective value for Fair-Lloyd is the maximum average cost of the demographic groups. The reported objective values are averaged over $20$ runs and the shaded areas are the standard deviation of these runs.}
\label{fig:convergencerate}
\end{figure*}

We used three datasets: 1) Adult dataset 
\citep{ucirepo}, consists of records of 48842 individuals collected from census data,  
with 103 features. 
The demographic groups considered are female/male for the 2-group setting and five racial groups of ``Amer-Indian-Eskim'', ``Asian-Pac-Islander'', ``Black'', ``White'', and ``Other'' for the multiple-groups setting; 2) Labeled faces in the wild (LFW) dataset \citep{huang2008labeled}, consists of 13232 images of celebrities. The size of each image is $49 \times 36$ or a vector of dimension $1764$.
The demographic groups are female/male; and
3) Credit dataset \citep{yeh2009comparisons}, consists of records of 30000 individuals with 21 features. We divided the multi-categorical education attribute to ``higher educated'' and ``lower educated'', and used these as the demographic groups.

As different features in any dataset have different units of measurements (e.g., age versus income), it is standard practice to normalize each attribute to have mean $0$ and variance $1$. 
We also converted any categorical attribute to numerical ones. For both Lloyd's and Fair-Lloyd we tried 200 different center initialization, each with 200 iterations. We used random initial centers (starting both algorithms with the same centers in each run). 

For clustering high-dimensional datasets with $k$-means, Principal Component Analysis (PCA) is often used as a pre-processing step~\citep{ding2004k,kumar2010clustering}, reducing the dimension to $k$. We evaluate Fair-Lloyd both with and without PCA. Since PCA itself could induce representational bias towards one of the (demographic) groups, 
Fair-PCA \citep{samadi2018price} has been shown to be an unbiased alternative, and we use it as a third pre-processing option. 
We refer to these three pre-processing choices as \textbf{w/o PCA}, \textbf{w/ PCA}, and \textbf{w/ Fair-PCA} respectively.

\begin{figure*}[t]
\begin{subfigure}{.335\textwidth}
  \centering
  \caption*{Credit}
  \includegraphics[width=\linewidth]{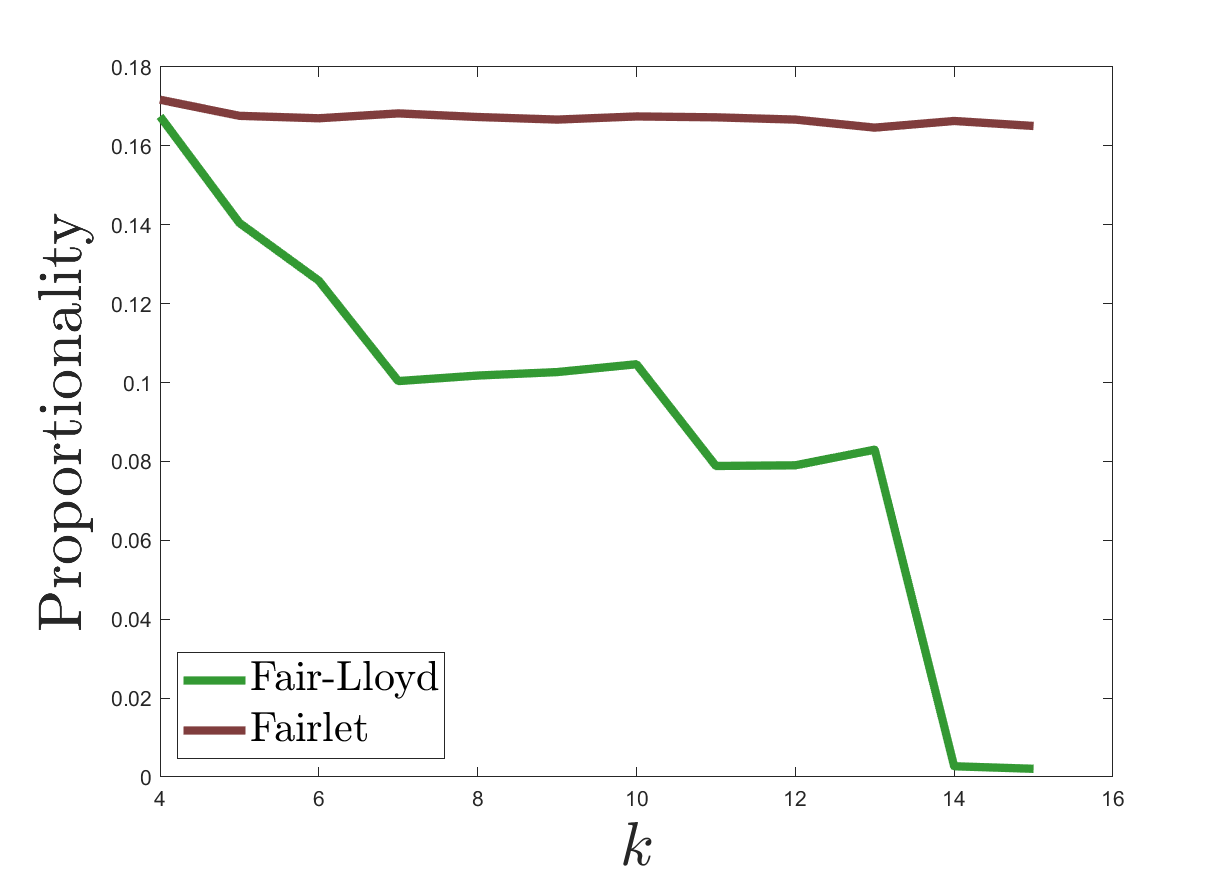}
\end{subfigure}
\hspace{0.7cm}
\begin{subfigure}{.29\textwidth}
  \centering
  \caption*{Adult}
  \includegraphics[width=\linewidth]{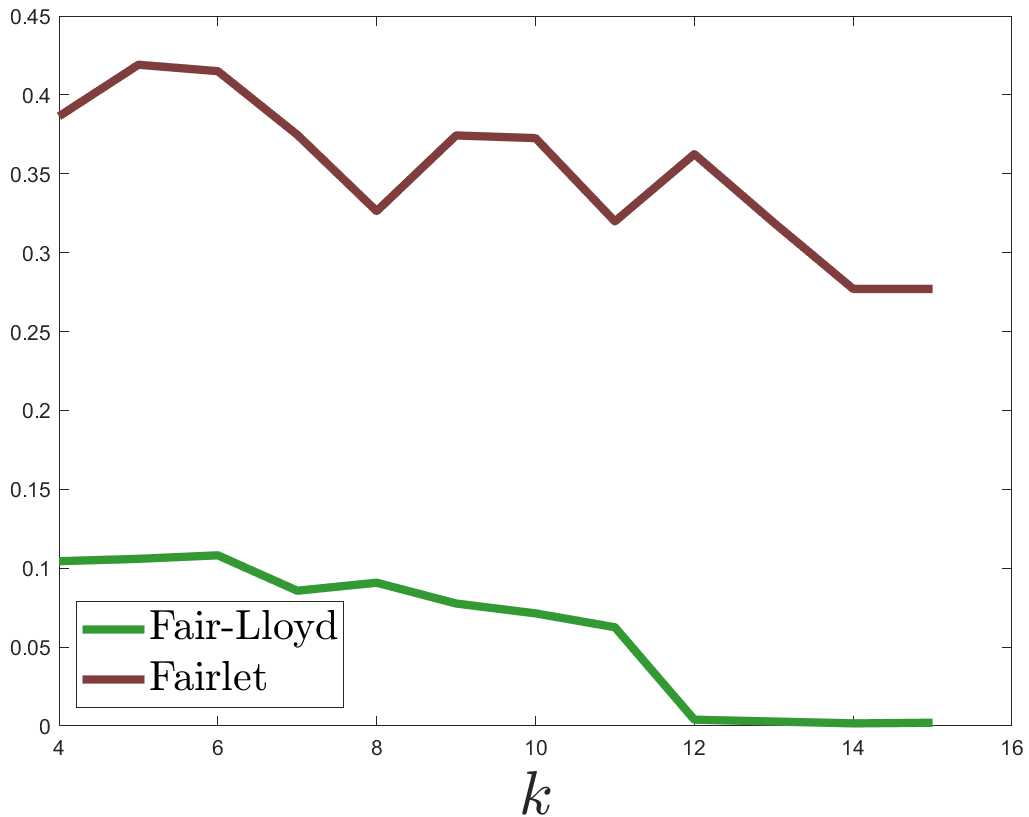}
\end{subfigure}
\begin{subfigure}{.335\textwidth}
  \centering
  \includegraphics[width=\linewidth]{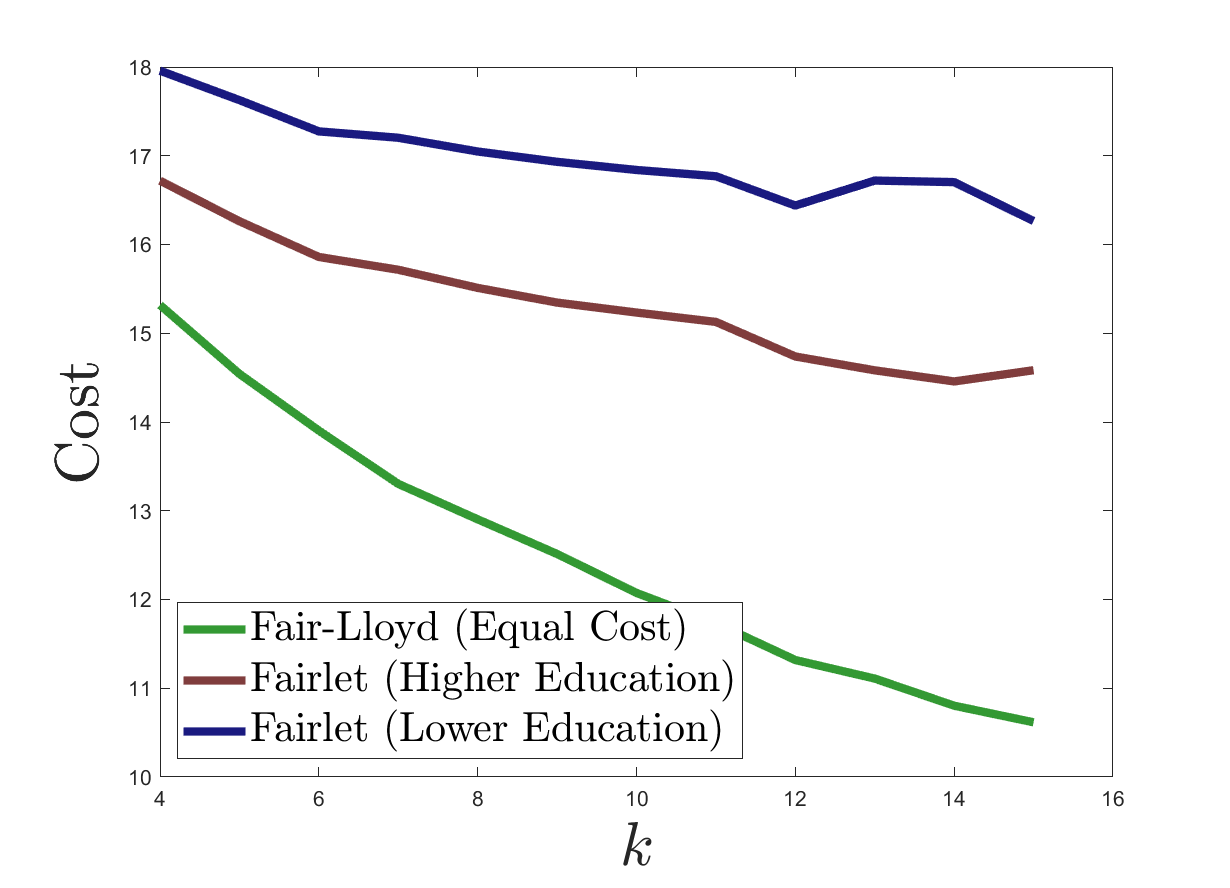} 
\end{subfigure}
\hspace{.7cm}
\begin{subfigure}{.29\textwidth}
  \centering
  \includegraphics[width=\linewidth]{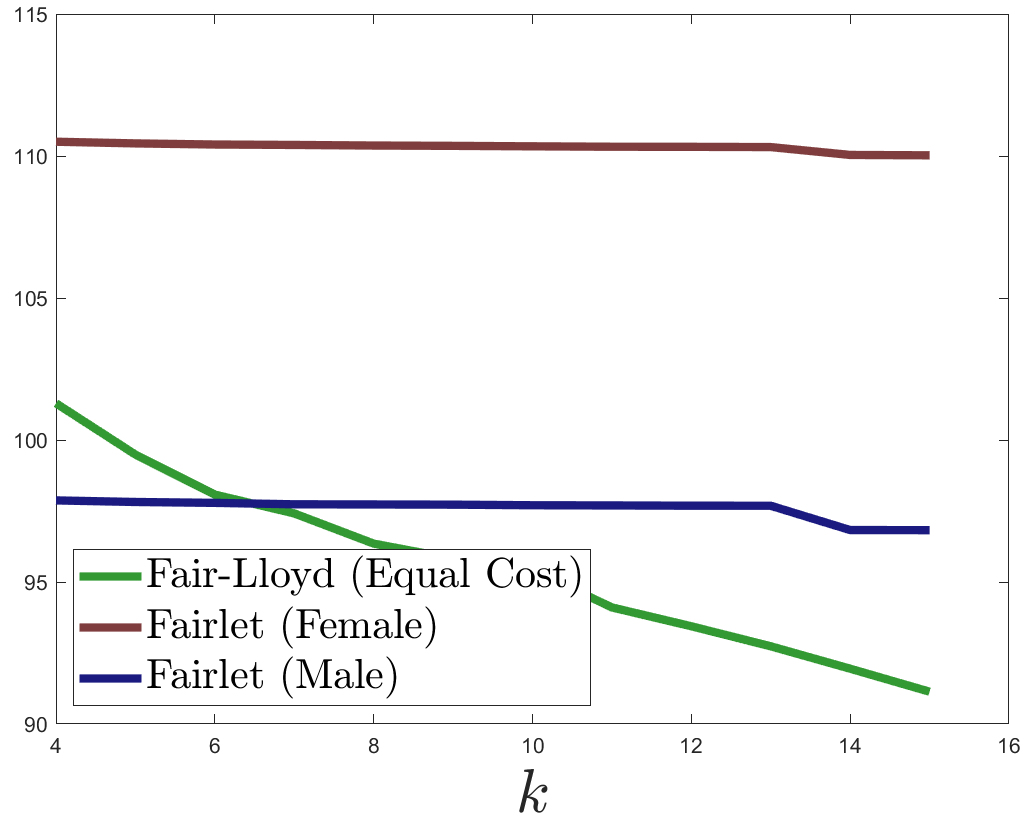} 
\end{subfigure}
\caption{Comparison of socially fair $k$-means (Fair-Lloyd) to proportionally fair $k$-means (Fairlet) on the Credit and Adult dataset in terms of proportionality and clustering cost.}
\label{fig:creditFairletComparison}
\end{figure*}

\paragraph{Results.}
Figure~\ref{fig:cost} shows the average clustering cost for different demographic groups.
In the first row, all datasets are evaluated in their original dimension with no pre-processing applied (w/o PCA). In the second and third rows (w/ PCA and w/ Fair-PCA), the PCA/Fair-PCA dimension is equal to the target number of clusters~$k$. 

Our first observation is that the standard Lloyd's algorithm results in a significant gap between the clustering cost of individuals in different groups, with higher clustering cost for females in the Adult and LFW datasets, and for lower-educated individuals in the Credit dataset. The average clustering cost of a female is up to 15\% (11\%) higher than a male in the Adult (LFW) dataset when using standard Lloyd's. A similar bias is observed in the Credit dataset, where Lloyd's leads up to 12\% higher average cost for a lower-educated individual compared to a higher-educated individual.

Our second observation is that the Fair-Lloyd algorithm
effectively eliminates this bias by outputting a clustering with equal clustering costs for individuals in different demographic groups. More precisely, for the Credit and Adult datasets the average costs of two demographic groups are identical, represented by the yellow line in Figure~\ref{fig:cost}.
For the LFW dataset, 
we observe a very small difference in the average clustering cost over the two groups in the fair clustering ($0.4\%$, $1\%$ and $0.6\%$ difference for without PCA, with PCA, and with Fair-PCA respectively). Notably, Fair-Lloyd mitigates the bias of the output clustering independent of whether it is applied on the original data space, on the PCA space, or on the Fair-PCA space. In Figure~\ref{fig:adultAll}, we show a snapshot of performance of Fair-Lloyd versus Lloyd's on the Adult dataset for all three different 
pre-processing
choices.

Figure~\ref{fig:adult5RacesRatio} shows the maximum ratio of average cost between any two racial groups in the Adult dataset, which comprised of five racial groups ``Amer-Indian-Eskim'', ``Asian-Pac-Islander'', ``Black'', ``White'', and ``Other''.
Note that, the max cost ratio of one indicates that all groups have the same average cost in the output clustering.
As we observe, the standard Lloyd algorithm results in a significant gap between the cost of different groups
resulting in a high max cost ratio overall.
As for the Fair-Lloyd algorithm, as the number of clusters increases, it outputs a clustering of the data with same average cost for all the demographic groups.

\vspace{-2mm}
\paragraph{The price of fairness.} Does requiring fairness come at a price, in terms of either running time or overall $k$-means cost?
Figure~\ref{fig:runtime} shows the running time of Lloyd's versus Fair-Lloyd 
for $200$ iterations. 
Running time for all three datasets is measured in the $k$-dimensional PCA space, where $k$ is the number of clusters. 
As we observe, Fair-Lloyd incurs a very small overhead in the running time, with only 
4\%, 4\%, and 8\%
increase (on average over $k$) for the Adult, Credit, and LFW dataset respectively. Moreover, as illustrated in Figure~\ref{fig:convergencerate}, the convergence rate of Lloyd and Fair-lloyd are essentially the same in practice.
Finally, 
the increase in the standard $k$-means cost of Fair-Lloyd solutions (averaged over the entire population 
)
was at most 
4.1\%, 2.2\% and 0.3\%
for the LFW, Adult, and Credit datasets, respectively. Arguably, this is outweighed by the benefit of equal cost to the two groups.

\newcommand{\sizeA}{0.3}
\newcommand{\sizeB}{.45}

\paragraph{Socially fair versus proportionally fair.} The first introduced notion of fairness for $k$-means clustering considered the proportionality of the sensitive attributes in each cluster \cite{chierichetti2017fair}. For the case of two groups $A$ and $B$ (e.g., male and female), and a clustering of points $\mathcal{U} = \{U_1, \ldots, U_k\}$, the proportionality or balance of the clustering $\mathcal{U}$ is formally defined as
\[
\min_{1\leq i\leq k} \min \{ \frac{|A\cap U_i|}{|B\cap U_i|}, \frac{|B\cap U_i|}{|A\cap U_i|}\}
\]
We emphasize that improving the proportionality is at odds with improving the maximum average cost of the groups. This can be seen in Figure~\ref{fig:motivation2}. To illustrate this more, we compared our method to one of the proposed methods that guarantees the proportionality of the clusters on the credit and adult datasets. We used the code provided in~\cite{bera2019fair}. As illustrated in Figure~\ref{fig:creditFairletComparison}, the proportionally fair method fails to achieve an equal average cost for different populations and our methods do not achieve proportionally fair clusters.

\section{Discussion}
Fairness is an increasingly important consideration for Machine Learning, including classification and clustering. Our work shows that the most popular clustering
method, Lloyd's algorthm,
can be made fair, in terms of average cost to each subgroup, with minimal increase in the running time or the overall average $k$-means cost, while maintaining its simplicity, generality and stability. 
Previous work on fair clustering focused on proportional representation of sensitive attributes within clusters, while we optimize the maximum cost to subgroups. As Figure~\ref{fig:motivation2} suggests, and Figure~\ref{fig:creditFairletComparison}
shows on benchmark data sets, these criteria lead to different solutions. We believe that both perspectives are important, and the choice of which clustering to use will depend on the context and application, e.g., proportional representation might be paramount for partitioning electoral precincts, while minimizing cost for every subgroup is crucial for resource allocation.

\newpage
\bibliography{facct}
\bibliographystyle{acm}

\clearpage
\section{Appendix}
The functions $f_j$ and their maximum are not necessarily quasiconvex in terms of $\gamma_j$'s. Figure~\ref{fig:qcexample} shows an example --- it illustrates a level set of function $F$. In this example, we have two clusters and three groups. The parameters used for this example are listed in the corresponding table.

\begin{minipage}[c]{\columnwidth}
\begin{figure}[H]
\begin{minipage}[c]{\columnwidth}
    \includegraphics[width=0.9\columnwidth]{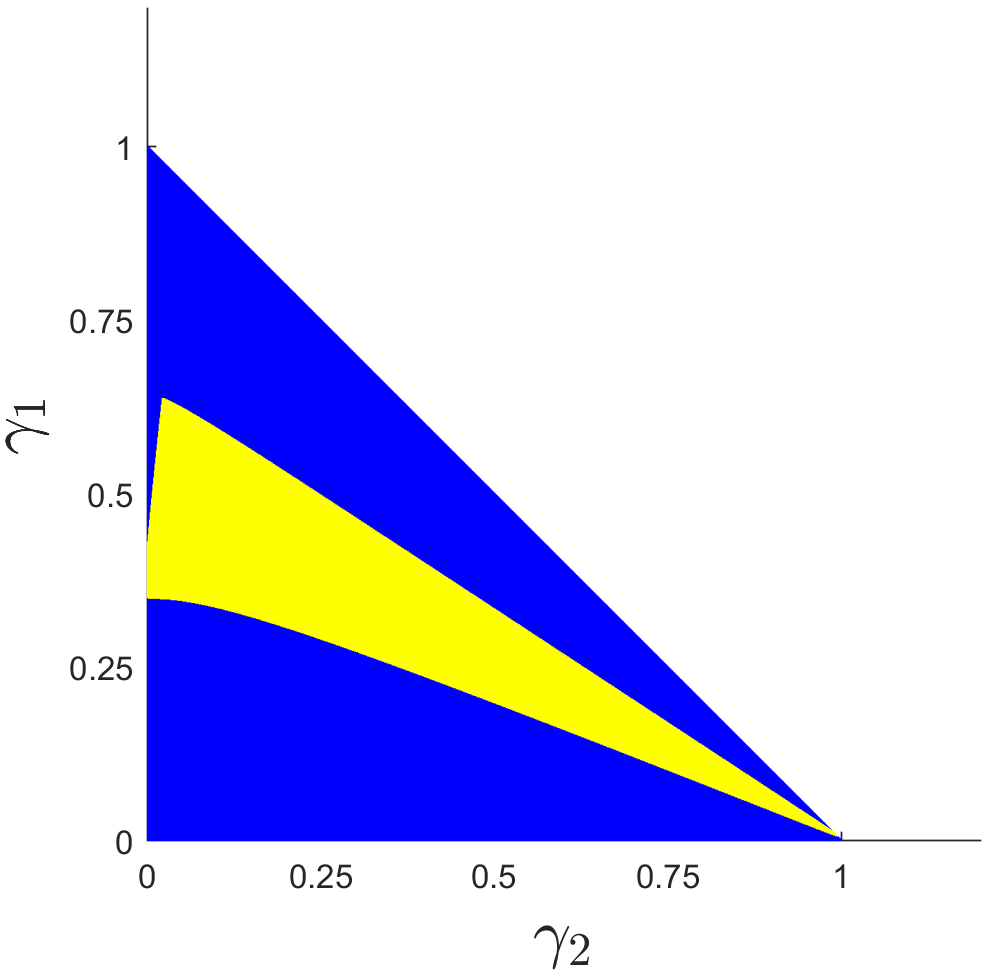}
  \end{minipage}
  \vspace{4mm}
  \begin{minipage}[c]{\columnwidth}
  \vspace{4mm}
  \centering
\begin{tabular}{|c|c|c|c|}
\hline & & & \\[-8pt]
$\alpha_i^j$                                    & $j=1$   & $j=2$   & $j=3$   \\[3pt] \hline & & & \\[-8pt]
$i=1$                                           & $0.9$   & $0.01$  & $0.95$  \\[3pt] \hline & & & \\[-8pt]
$i=2$                                           & $0.1$   & $0.99$  & $0.05$  \\[3pt] \hline \multicolumn{4}{c}{}\\[-5pt] \hline  & & & \\[-8pt]
      & $j=1$   & $j=2$   & $j=3$   \\[3pt] \hline  & & & \\[-8pt]
$\frac{\Delta(M^j,\mathcal{U}\cap A_j)}{|A_j|}$ & $0$     & $1$     & $0.1$   \\[6pt] \hline \multicolumn{4}{c}{}\\[-5pt] \hline  & & & \\[-8pt]
$\mu_i^j$ & $j=1$   & $j=2$   & $j=3$   \\[3pt] \hline  & & & \\[-8pt]
$i=1$                                           & $(0,0)$ & $(2,2)$ & $(3,1)$ \\[3pt] \hline  & & & \\[-8pt]
$i=2$                                           & $(0,0)$ & $(2,2)$ & $(3,1)$ \\[3pt] \hline
\end{tabular}
\end{minipage}
  \caption{An example that shows $F$ is not quasiconvex. The yellow area represents the points for which the value of $F$ is less than $4.2$. As one can see, the yellow area is not convex and therefore $F$ is not quasiconvex in terms of $\gamma_j$'s. The table shows the parameters that were used for this example.}
  \label{fig:qcexample}
\end{figure}
\end{minipage}
\end{document}